\documentclass{article}

\usepackage[preprint]{neurips_2022}

\usepackage[utf8]{inputenc} 
\usepackage[T1]{fontenc}    
\usepackage{hyperref}       
\usepackage{url}            
\usepackage{booktabs}       
\usepackage{amsfonts}       
\usepackage{nicefrac}       
\usepackage{microtype}      
\usepackage{xcolor}         
\usepackage{amsthm}

\usepackage{amsmath,amssymb, graphicx}
\usepackage{subcaption}

\newcommand{\mathd}{\mathrm{d}}

\newtheorem{assumption}{Assumption}
\newtheorem{theorem}{Theorem}
\newtheorem{lemma}{Lemma}
\newtheorem{definition}{Definition}
\newtheorem{proposition}{Proposition}
\newtheorem*{mainresult}{Main result}
\newtheorem{corollaryun}{Corollary}

\title{Score-Based Generative Models Detect Manifolds}

\author{%
  Jakiw Pidstrigach \\
  Institut f\"ur Mathematik\\
  Universit\"at Potsdam\\
  Karl-Liebknecht-Str. 24/25 \\
  14476 Potsdam \\
  \texttt{pidstrigach@mailbox.org} \\
}

\begin{document}

\maketitle

\begin{abstract}
	Score-based generative models (SGMs) need to approximate the scores $\nabla \log p_t$ of the intermediate distributions as well as the final distribution $p_T$ of the forward process. The theoretical underpinnings of the effects of these approximations are still lacking. We find precise conditions under which SGMs are able to produce samples from an underlying (low-dimensional) \emph{data manifold} $\mathcal{M}$. This assures us that SGMs are able to generate the ``right kind of samples''. For example, taking $\mathcal{M}$ to be the subset of images of faces, we find conditions under which the SGM robustly produces an image of a face, even though the relative frequencies of these images might not accurately represent the true data generating distribution. 
	Moreover, this analysis is a first step towards understanding the generalization properties of SGMs: Taking $\mathcal{M}$ to be the set of all training samples, our results provide a precise description of when the SGM memorizes its training data.
\end{abstract}

\section{Introduction}\label{intro}

Score-based generative models, also called diffusion models ({\citep{sohl2015deep, DBLP:conf/nips/SongE19, DBLP:conf/iclr/0011SKKEP21, vahdat2021score}) and the related models (\citep{bordes2017learning, ho2020denoising, kingma2021variational}) have shown great empirical
	success in many areas, such as image generation (\citep{DBLP:conf/iclr/Jolicoeur-Martineau21, nichol2021improved, dhariwal2021diffusion, ho2022cascaded}), audio generation
	(\citep{DBLP:conf/iclr/ChenZZWNC21, DBLP:conf/iclr/KongPHZC21, DBLP:conf/interspeech/JeongKCCK21, DBLP:conf/icml/PopovVGSK21}) as well as in other applications
	(\citep{DBLP:journals/corr/abs-2111-13606, de2021diffusion, DBLP:conf/iccv/ZhouD021, DBLP:conf/eccv/CaiYAHBSH20, DBLP:conf/cvpr/LuoH21, DBLP:journals/corr/abs-2108-01073, DBLP:journals/corr/abs-2104-07636, DBLP:journals/ijon/LiYCCFXLC22, DBLP:journals/corr/abs-2104-05358}). Recently some progress has been made
	to bridge the gap between the different approaches
	({\citep{DBLP:conf/iclr/0011SKKEP21, huang2021variational}}) through the framework of SDEs and reverse SDEs.
	
	In generative modelling one is given samples $\{x^i\}_{i=1}^N$ from a measure $\mu_\text{data}$. The task is to learn a measure $\mu_\text{sample}$ which approximates $\mu_\text{data}$. The performance of a generative model can then be measured by the distance from $\mu_\text{sample}$ to $\mu_\text{data}$. In practice however, the true data generating distribution $\mu_\text{data}$ is unknown. All that is known are the samples $\{x_i\}_{i=1}^n$, which can be used to define the empirical measure $\hat{\mu}_\text{data}$,
	\[
		\hat{\mu}_\text{data} := \text{Unif}\{x^1, x^2, \ldots, x^n\}.
	\]
	Any sample from the empirical measure $\hat{\mu}_\text{data}$ will be equal to a training example. Hence, while $\mu_\text{sample}$ being close to $\mu_\text{data}$ is the final goal, $\mu_\text{sample}$ being close to $\hat{\mu}_\text{data}$ implies that the generative model has memorized the training data. To summarize, a good generative model will output a measure $\mu_\text{sample}$ which is as close to $\mu_\text{data}$ as possible, while keeping some distance from $\hat{\mu}_\text{data}$, even though it only knows $\mu_\text{data}$ through $\hat{\mu}_\text{data}$. 
	
	Given a target measure $\pi_0$, a score-based generative model (SGM) employs two stochastic differential equations (SDEs). The first one is called the \emph{forward SDE}
	\begin{equation}
		\label{diffusingprocesscont}
		\begin{array}{lll}
			\mathd X_t & = & \beta (X_t) \mathd t + \sigma  \mathd W_t,\\
			X_0 & \sim & \pi_0.
		\end{array} 
	\end{equation}
	The marginals of $X_t$ are denoted by $\pi_t$. The forward SDE is run until some terminal time $T$. Furthermore, the \emph{reverse SDE} is defined by
	\begin{equation}
		\begin{array}{lll}
			\mathd Y_t & = & - \beta (Y_t) \mathd t + \sigma \sigma^T \nabla \log p_{T - t} (Y_t) \mathd t + \sigma  \mathd B_t,\\
			Y_0 & \sim & q_0.
		\end{array} \label{reversesde}
	\end{equation}
	We refer to the marginals of $Y_t$ as $q_t$. The samples are generated from the final distribution $q_T$, i.e. $\mu_\text{sample} := q_T$. The reverse SDE has the property that if $q_0$ is chosen to be equal to $\pi_T$, then $q_t = \pi_{T-t}$. In particular, this implies that $q_T = \pi_0$. Therefore, if we have samples from $\pi_T$, we can run the reverse SDE on them to create new samples from $\pi_0$.
	
	In the following we will denote by $p_t$ the marginals of the forward SDE when started in the true data generating distribution $\pi_0 = \mu_\text{data}$. We will denote by $\hat{p}_t$ the marginals of the forward SDE when started in $\pi_0 = \hat{\mu}_\text{data}$. Optimally, we would like to run the algorithm using $\pi_0 = \mu_\text{data}$, i.e. with marginals $\pi_t = p_t$. This is however not possible, since $\mu_\text{data}$ itself is unknown.
	
	To circumvent the problem of not knowing $p_T$, the forward SDE is chosen such that it forgets its initial condition $p_0$. At time $T$, the marginal $p_T$ is then well approximated by a proxy distribution $\mu_\text{prior} \approx p_T$, independently of $p_0$. Additionally, the marginals $p_t$ and therefore the scores $\nabla \log p_t$ cannot be evaluated for $p_0 = \mu_\text{data}$. Therefore, the scores are replaced by a neural network $s_\theta(x,t)$, which is trained via score-matching techniques \citep{vincent2011connection, hyvarinen2005estimation}.

	\begin{figure}
		
		\begin{subfigure}[T]{0.69\textwidth}
			\includegraphics[width=\textwidth, height=5.5cm]{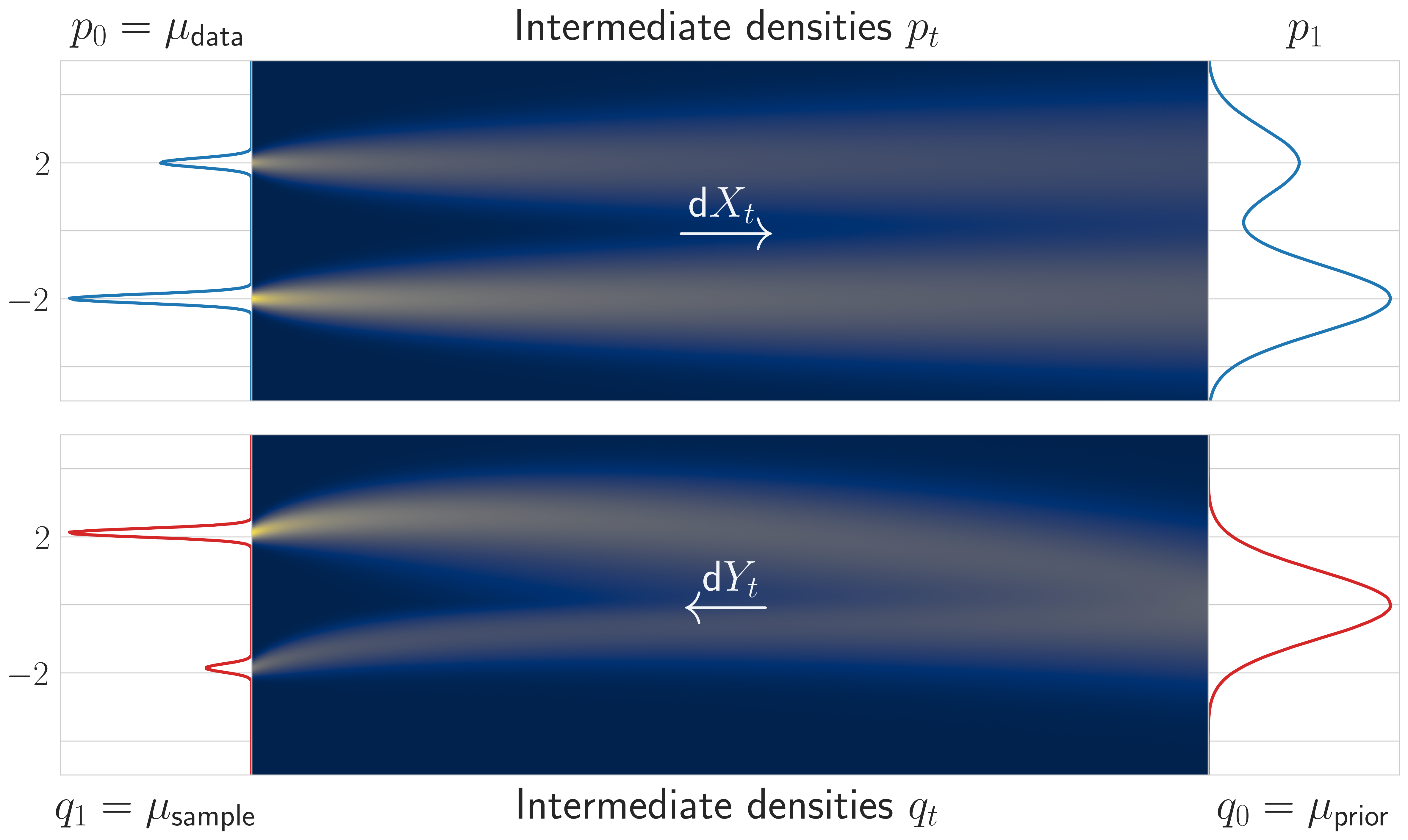}
			\label{fig:density_flow}
		\end{subfigure}
		\begin{subfigure}[T]{0.3\textwidth}
			\hspace{200cm}
			\includegraphics[width=\textwidth, height=4cm]{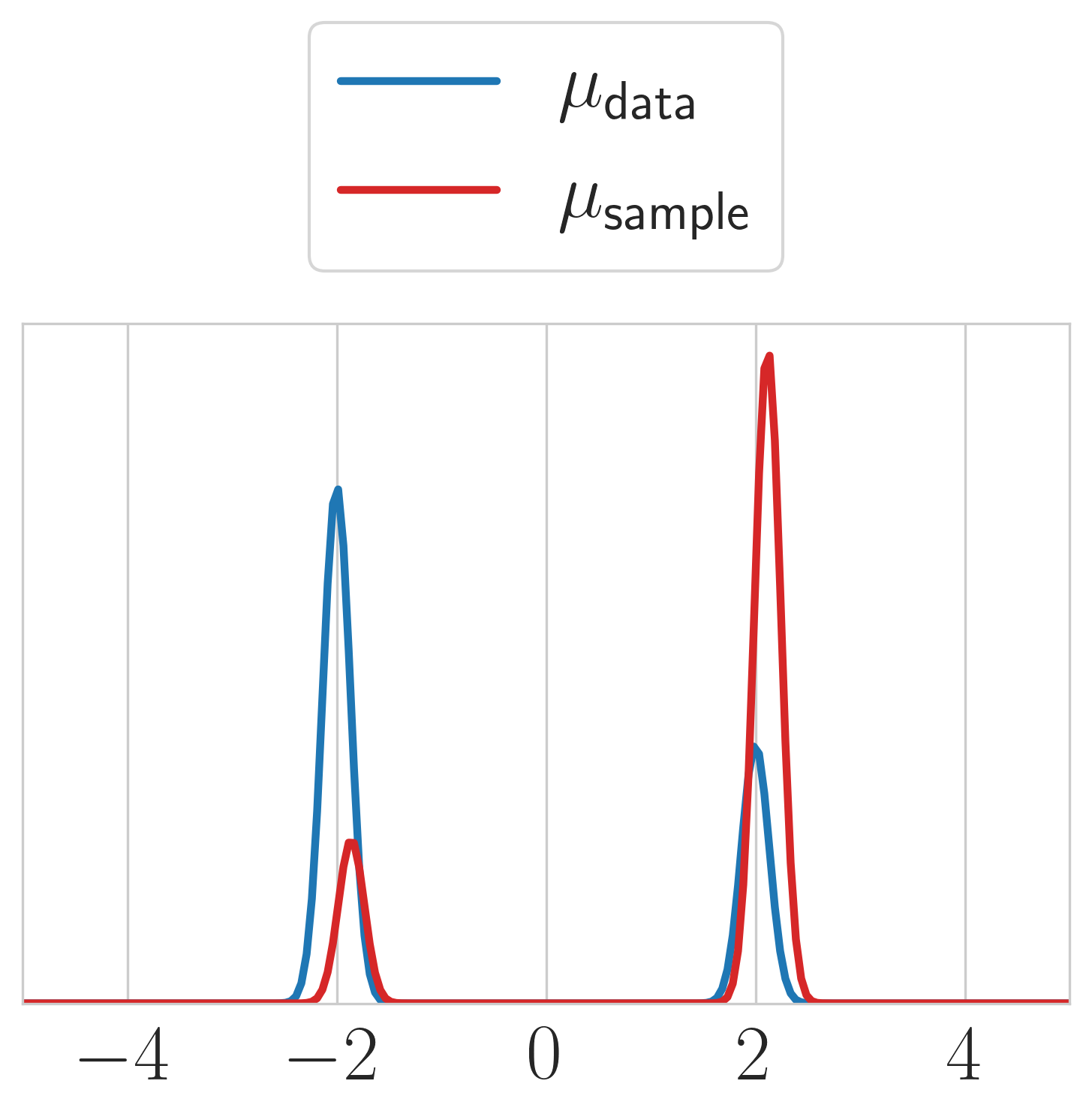}
			\label{fig:comparison_final}
			
		\end{subfigure}
		
		\caption{
			\emph{Top left}: The leftmost plot shows the true data distribution $\mu_\text{data}$ which is a Gaussian mixture.  The heat maps show the intermediate densities $p_t$ of $X_t$, followed by line plots of $p_t$ for $t=1$.\\
			\emph{Bottom left}: The rightmost plot shows $\mu_\text{prior}$, which is a standard Gaussian and differs from $p_1$. We start the reverse SDE \eqref{reversesde} in $\mu_\text{prior}$. But instead of using the real score, we introduce an approximation error and use $s(x,t) = \nabla \log p_{1-t}(x) + 3$ with a constant error of $3$. Again, the heat maps show how the densities $q_t$ of $Y_t$ evolve backwards in time. The leftmost plot shows the resulting distribution $q_1$ which is used as sample distribution, $\mu_\text{sample} = q_1$.\\
			\emph{Right:} The densities $\mu_\text{data}$ and $\mu_\text{sample}$ are shown for direct comparison. We see that the approximation errors in $\mu_\text{prior}$ and the drift lead to an incorrect sample distribution $\mu_\text{sample} \not= \mu_\text{data}$. Nevertheless, $\mu_\text{sample}$ is supported in the same area as $\mu_\text{data}$.
			For details on the numerical implementation see Appendix \ref{sec:numerics}.
		}
		\label{fig:forward_backward_problem}
	\end{figure}

	As a result, SGMs make two approximations. The first one is in approximating $p_T$ by $\mu_\text{prior}$. The second one is the approximation of $\nabla \log p_t$ by the neural net $s_\theta(x, t)$. We illustrate this in Figure \ref{fig:forward_backward_problem}. It is important to understand how these approximations translate to the distance of $\mu_\text{sample}$ to $\mu_\text{data}$ or $\hat{\mu}_\text{data}$. 
	
	Some early works already deal with these questions. In \citep{de2021diffusion} the total variation distance between $\mu_\text{sample}$ and $\mu_\text{data}$ in bounded, whereas \citep{song2021maximum} derives bounds with respect to the KL-Divergence. The work \citep{de2021diffusion} furthermore derives a result similarly to the second part of Theorem \ref{thm:mu_prior}, but also treating the errors that are introduced by discretizing the SDE. However, both of these works assume that the initial distribution $\mu_\text{data}$ is rather well behaved. In particular, it is assumed that $\mu_\text{data}(x) > 0$ for all $x$. We shortly discuss this assumption now.
	
	Assuming that $\mu_\text{data}(x) > 0$ for any $x$ means that one postulates that every $x$ is a possible sample from $\mu_\text{data}$. For example, this implies that even if all samples $\{x_1, \ldots, x_n\}$ of $\mu_\text{data}$ consist of images of human faces, we say that $\mu_\text{data}$ can possibly also generate images of for example furniture, animals or pure white noise. Even though it might have low probability, any combinations of pixels is a possible sample from $\mu_\text{data}$. The assumption that $\mu_\text{data}$ is actually supported on some lower dimensional substructure $\mathcal{M}$ is well known under the name \emph{manifold hypothesis}, see for example \citep{bengio2013representation, DBLP:conf/iclr/PopeZAGG21}. In practice this means that $\mu_\text{data}(x) = 0$ for many $x$. This also leads to exploding scores $\nabla \log p_t$ as $t \to 0$ (see Section \ref{sec:drift_explosion}), a behaviour which also observed empirically and further underpins the relevance of the manifold assumption. We will from now on denote the support of $\mu_\text{data}$ as \emph{data manifold} $\mathcal{M}$. 

	A fundamental question is then: What is the support of $\mu_\text{sample}$ and how does it compare to $\mathcal{M}$. This is interesting for multiple reasons. On the one hand, if we for example assume that $\mathcal{M}$ is the set of all images of faces, then the knowledge that $\mu_\text{sample}$ also has support $\mathcal{M}$ implies that, regardless of how close $\mu_\text{sample}$ actually is to $\mu_\text{data}$, it will at least always produce an image of a face. On the other hand, we can also compare the support of $\mu_\text{sample}$ to that of $\hat{\mu}_\text{data}$. The measures $\mu_\text{sample}$ and $\mu_\text{data}$ sharing their support translates to $\mu_\text{sample}$ memorizing the training data and not being able to generalize. Both of these are very valuable insights into the qualities of $\mu_\text{sample}$ on its own. Furthermore, statistical distances like the KL-Divergence or the total variation distance are only meaningful if the support of the measures overlap to some degree, as otherwise they will equal the maximum distance value.
	
	\begin{figure}
		\begin{subfigure}{\linewidth}
			\centering
			\includegraphics[width=\linewidth]{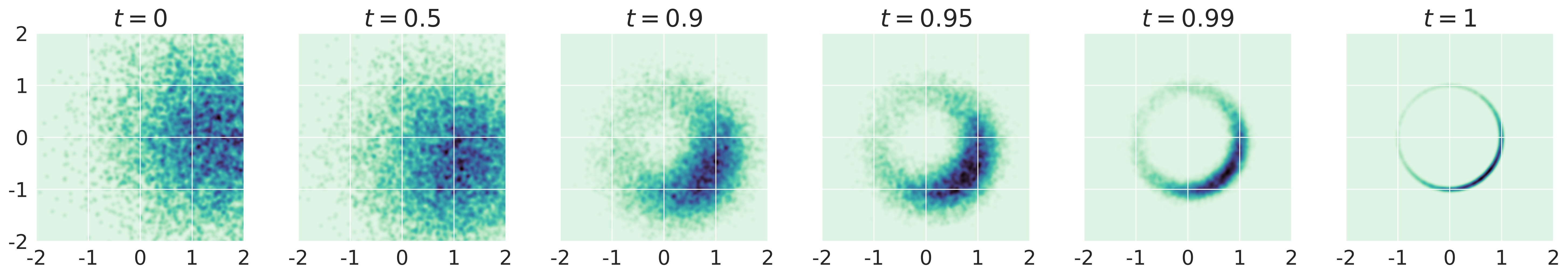}
			\caption{Here ${\mu}_\text{data}$ is chosen as the uniform distribution on the unit sphere $\mathcal{S}^1$. We use the Brownian motion for the forward SDE. We can compute the exact score $\nabla \log p_t(x)$. We perturb it with the vector $v = (x=0, y=-1)$ and define the approximation $s(x,t) = \nabla \log p_t(x) + v$. Furthermore, we purposely adopt a poor approximation $\mu_\text{prior}\not\approx p_1$ by setting $\mu_\text{prior} = \mathcal{N}(m, I)$, where $m=(x=1.5, y=0)$. We then run the reverse SDE \eqref{reversesdewitherror}. The figure shows heat maps of the intermediate distributions $q_t$ of the reverse SDE. At time $t=1$ we reach $q_1 = \mu_\text{sample}$. We see that $\mu_\text{sample}$ is a distribution on $\mathcal{M} = \mathcal{S}^1$, albeit not the uniform one. Furthermore, we can observe how the errors in the initial conditions and the drift influence the skewed distribution $\mu_\text{sample}$. The initial conditions where chosen to have a to large $x$-coordinate on average, whereas the drift was chosen as to prefer lower $y$-coordinates. The distribution $\mu_\text{sample}$ is concentrated in areas with high $x$ and low $y$-coordinates.}
			\label{fig:sphere_both_disturbed}
		\end{subfigure}
		\begin{subfigure}{\linewidth}
			\centering
			\includegraphics[width=\linewidth]{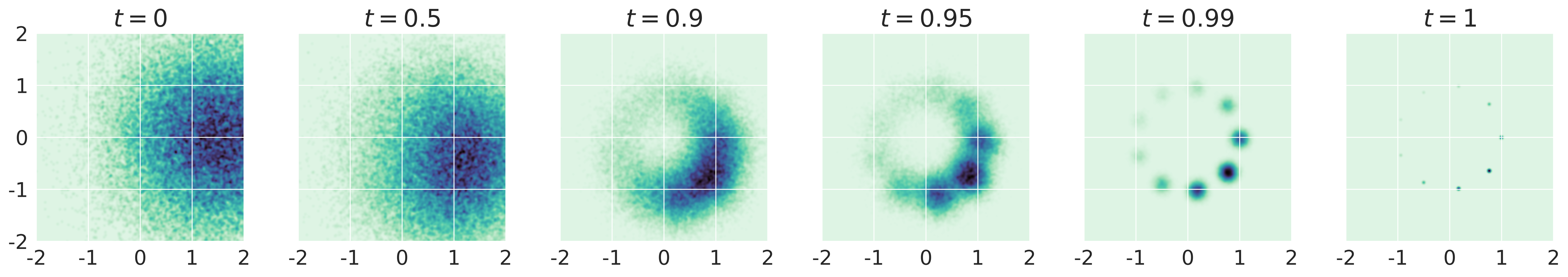}
			\caption{The above experiment is repeated but with ${\mu}_\text{data}$ chosen to be the uniform distribution on $\mathcal{M} = \{x_i\}_{i=1}^9$, where the $x_i$ are $9$ evenly spaced points on the unit sphere $\mathcal{S}^1$. Notice that while the approximation errors cause a non-uniform distribution the training examples, $\mu_\text{sample}$ is still supported solely on $\mathcal{M}$ and will not generate novel samples.}
			\label{fig:sphere_samples_both_disturbed}
		\end{subfigure}
		\caption{Perturbing $\nabla \log \hat{p}_t$.}
		\label{fig:sphere_example}
	\end{figure}
	
	If two measures have the same support, they are said to be \emph{equivalent}. Our main result is the following:
	\begin{mainresult}
		We identify conditions under which $\mu_\text{sample}$ is able to learn the data manifold, that is $\text{supp}(\mu_\text{sample}) = \text{supp}(\mu_\text{data})$. Applying these results to different settings, we find precise conditions under which an SGM memorizes its training data or under which it is able to learn the right data manifold $\mathcal{M}$.
		
		In particular we find that for SGMs to be able to generalize, the approximation error made when approximating the training drift $\nabla \log \hat{p}_t$ has to be unbounded.
	\end{mainresult}
	A first illustration of these results is given in Figure \ref{fig:sphere_example}. We use a simple example, where we can perfectly evaluate the true drift $\nabla \log \pi_t$. We then choose an incorrect initial condition $\mu_\text{prior}$ which is far from $p_T$, and also add a constant error to $\nabla \log \pi_t(x)$. The initial measures $\pi_0$ are given as the uniform distribution on the unit sphere and the uniform distribution on 9 samples from the unit sphere in Figure \ref{fig:sphere_both_disturbed} and Figure \ref{fig:sphere_example} respectively. We see that the final distribution of the reverse SDE, $\mu_\text{sample}$, is not the uniform distribution anymore. This is due to errors in the initial conditions and the drift. Nevertheless, $\mu_\text{sample}$ is still supported on the exact same subset as $\pi_0$.
	
	When applying our main result to the empirical measure $\hat{\mu}_\text{data}$ one gets the following corollary, supplying precise conditions under which a SGM has memorized its training data.
	\begin{corollaryun}
		\label{cor:generalization}
		Denote by $\hat{X}_t$ the forward SDE when started in the empirical measure $\pi_0 = \hat{\mu}_\text{data}$.
		Let $\int_0^T \| s_\theta(\hat{X}_t, t) - \nabla \log \hat{p}_t(\hat{X}_t) \mathrm{d}t$ be drift approximation error along a path of the forward SDE. For a given weighting function $w(t)$, the training objective of an SGM can be written as
		\begin{equation}
			\label{equ:approx_drift_l2_integral}
			L_2 = \mathbb{E}_{\hat{X}}[\int_0^T w(t) \| s_\theta(\hat{X}_t, t) - \nabla \log \hat{p}_t(\hat{X}_t) \|^2 \mathrm{d}t],
		\end{equation}
		see Section \ref{sec:score_approximation}. Simultaneously, if the exponential integral of the drift approximation error is integrable in the following sense:
		\begin{equation}
			\label{equ:approx_drift_exp_integral}
			L_{\exp} = \mathbb{E}_{\hat{X}}[\exp(\frac{\sigma}{2} \int_0^T \| s_\theta(\hat{X}_t, t) - \nabla \log \hat{p}_t(X_t) \|^2 \mathrm{d}t)] < \infty,
		\end{equation}
		the SGM has memorized its training data. Therefore, while training an SGM one should aim to minimize the mean squared error \ref{equ:approx_drift_l2_integral} while ensuring that the mean exponential error stays infinite, $L_{\exp} = \infty$. 
	\end{corollaryun}
	In particular, if $\| s_\theta(\hat{X}_t, t) - \nabla \log \hat{p}_t(X_t) \|$ is bounded, then $L_{\exp}$ is finite. Therefore, the the generalization capability of a SGM crucially depends on the training error being unbounded.
	
	We now proceed as follows. In Section \ref{currentmethods} we will summarize some of the most popular forward SDEs that are applied in SGMs. Then, in Section \ref{sec:score_approximation} we discuss how the drift approximation $s_\theta(x,t)$ for SGMs is trained in most implementations. We are ready to state our main results in Section \ref{ref:main_results}. In Section \ref{sec:drift_explosion} we show how the empirically observed drift explosion is related to the the manifold hypothesis. Most of the paper discusses the error in the drift approximation. But as we have discussed, we also have an error in the initial conditions. In Section \ref{sec:approximate_p_T} we will discuss how large the error in the initial conditions will be in practice.

	\section{Popular SDEs used in SGMs}\label{currentmethods}
	In this section we introduce some of the most popular SDEs that are used when implementing SGMs.
	The first works on SGMs studied discrete forward and backward processes. Nevertheless, the transition kernels and algorithms proposed in those works can be seen as discretisations of some well-known SDEs. More recent works have studied this connection and state the algorithms in terms of SDEs ({\citep{DBLP:conf/iclr/0011SKKEP21, huang2021variational}}). 
	
	\paragraph{Brownian Motion:} The works \citep{DBLP:conf/nips/SongE19, song2020improved} can be seen as a discretization of the SDE
	\[ \mathd X_t = \sigma(t) \mathd W_t. \]
	Denoting $h(t) = \int_0^t \sigma(s) \mathd s$, the solution to the above process can be explicitly stated as a time-changed Brownian motion, $X_t = W_{h(t)}$. The time-change can help in the implementation but does not alter the qualitative behaviour of the reverse SDE. In our following analysis we therefore set $\sigma(t) = 1$. Nevertheless, our results still hold for any positive $\sigma(t)$.
	
	\paragraph{Ornstein-Uhlenbeck Process:}The works {\citep{sohl2015deep, ho2020denoising}} can be seen as a discretization of
	\[ \mathd X_t = - \frac{1}{2} \alpha(t) X_t \mathd t + \sqrt{\alpha(t)} \mathd W_t , \]
	which is an Ornstein-Uhlenbeck process. Again, the parameters $\alpha_t$ are a time-change and do not influence the properties that we are investigating in this paper. Therefore, to simplify notation, we again set $\alpha(t)=1$. 
	
	\paragraph{Critically Damped Langevin Dynamics (CLD):} The work {\citep{DBLP:journals/corr/abs-2112-07068}} studies a second order SDE. Here artificial velocity
	coordinates $V_t$ are introduced and the system under consideration is
	\begin{eqnarray*}
		\mathd X_t & = &  V_t,\\
		\mathd V_t & = & - X_t - 2 V_t + 2 \mathd W_t,
	\end{eqnarray*}
	where $X_0 \sim p_\text{data}$, $V_0 \sim \mathcal{N}(0, I)$. For generation one runs the reverse SDE in $X_t$ and $V_t$ but discards the $V$ coordinate at the end.
	The work {\citep{DBLP:journals/corr/abs-2112-07068}} also includes the parameters $M$ and $\gamma$. We set both to $1$ as they do in their numerical experiments.
	
	\section{Score approximation with a finite number of samples}\label{sec:score_approximation}
	We now quickly discuss how the neural network is trained to approximate $\nabla \log p_t$ and what implications this has.
	The score $\nabla \log p_t$ is approximated by minimizing a variant of
	\begin{equation}
		L(\theta) = \int_0^T w(t) ~\mathbb{E}_{p_t(x)}[\|\nabla \log p_t(x) - s_{\theta}(x, t)\|^2]\mathrm{d}t.
		\label{equ:loss_score}
	\end{equation}
	for some weighting function $w$. The optimization is done via score matching techniques (see \citep{hyvarinen2005estimation, vincent2011connection, song2020sliced}). However, the above expectation depends on $p_t$, which we cannot evaluate since it depends on $p_0 = \mu_\text{data}$. Nevertheless, we can evaluate the approximation $\hat{p}_t$,
	\begin{equation}
		\hat{p}_t(x) = \mathbb{E}_{\hat{\mu}_\text{data}(x_0)}[p_{t|0}(x | x_0)] = \frac{1}{N} \sum_{i
		= 1}^N p_{t| 0} (x|x^i) \approx p_t(x) = \mathbb{E}_{\mu_\text{data}}[p_{t|0}(x | x_0)].
		\label{equ:hat_pi_definition}
	\end{equation}
	The surrogate loss
	\begin{equation}\label{equ:surrogate_loss}
		\hat{L}(\theta) = \int_0^T w(t) ~\mathbb{E}_{\hat{p}_t(x)}[\|\nabla \log \hat{p}_t(x) - s_{\theta}(x, t)\|^2] \mathrm{d}t,
	\end{equation}
	can be evaluated and is used for training. The equation \eqref{equ:surrogate_loss} is equivalent to \eqref{equ:approx_drift_l2_integral} since $\hat{X}_t$ has distribution $\hat{p}_t$. If we would minimize this loss perfectly, then $s_\theta(x,t)$ would be equal to $\nabla \log \hat{p}_t$. The reverse SDE started in an appropriate initial condition with drift $\nabla \log \hat{p}_t$ however, will produce samples from $\hat{\mu}_\text{data}$, which are training examples. This raises the question in which way or to which degree the loss should actually be minimized.

	\begin{figure}
		
		\begin{subfigure}[T]{0.6\textwidth}
			\includegraphics[width=\textwidth, height=5.5cm]{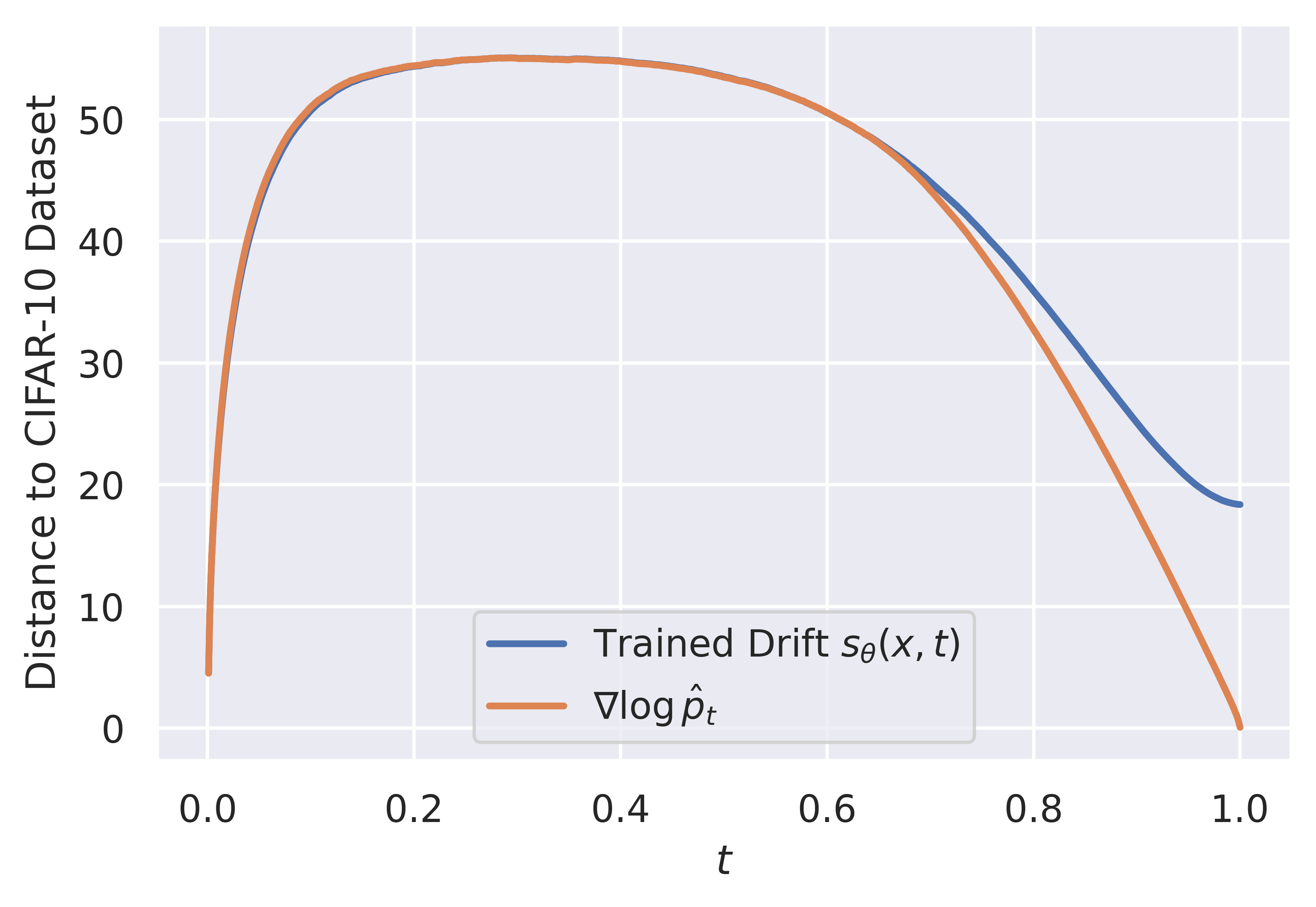}
			\caption{}
			\label{fig:distance_to_cifar}
		\end{subfigure}
		\begin{subfigure}[T]{0.33\textwidth}
			\includegraphics[width=\textwidth, height=4cm]{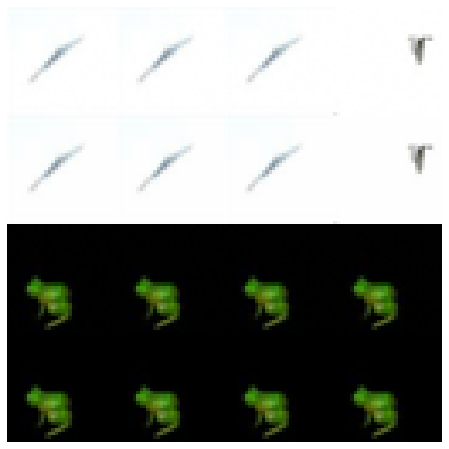}
			\caption{}
			\label{fig:cifar_with_drift_error}
		\end{subfigure}
		\caption{(a): Both lines correspond to the same experiment for different drifts in the reverse SDE. For both lines we started $N=1000$ paths in the zero vector in $\mathbb{R}^{32\times32\times3}$. For the blue line we used the pretrained CIFAR-10 DDPM++ model from \citet{DBLP:conf/iclr/0011SKKEP21}, whereas for the orange line we used the true drift $\nabla \log \hat{p}_t$, which is a mixture of $50 000$ Gaussians, one for each training example in CIFAR-10. We then saved the distance from $Y_t$ to the CIFAR-10 training examples, by calculating the distance to the closest example. Above we plot the average distance. We see, that while the reverse SDE run with $\hat{p}_t$ will have a distance of $0$ to the training examples in the end, the SDE with the pretrained drift keeps some distance to the training examples and therefore produces novel images.
		(b): We evaluate $\nabla \log \hat{p}_t$ as in (a). We do the analogous experiment to Figure \ref{fig:sphere_samples_both_disturbed} on CIFAR-10 and perturb the empirical drift $\nabla \log \hat{p}_t$ with a constant error vector. The first row shows the samples generated by adding the constant error vector $e(x, t) = (1, 1, \ldots, 1)\in \mathbb{R}^{32 \times 32 \times 3}$ to $\nabla \log \hat{p}_t$. In the second row we searched for the closest image in the CIFAR-10 dataset (with respect to the Euclidean $2$-distance on $\mathbb{R}^{32 \times 32 \times 3}$) and plotted it. We see that all the sampled images are nearly equal to a corresponding image in CIFAR-10. The distance of the images to their closest image in CIFAR-10 is around $0.07$ for all plotted images. Similar to Figure \ref{fig:sphere_example}, we can observe the effect of adding the one-vector. The sample distribution $\mu_\text{sample}$ got skewed to prefer images that have high pixel values. This corresponds to samples which are mostly white for the human eye. In the third and forth row we repeat the experiment of the first and second row, but add the negative one-vector $e(x,t) = (-1, -1, \ldots, -1)$ and get black images.}
		\label{fig:cifar_distance_and_drift_error}
	\end{figure}
	\section{Effects of the approximations}\label{ref:main_results}
	This section contains our main results. We first state the assumptions we have to make and then the Theorems.
	\subsection{Error in the initial condition}
	The following assumption is needed for the reverse SDE to be defined even in the case when the initial distribution $\pi_0$ has a degenerate support. In Lemma \ref{lemma:assumptions_hold} we will show that all the SDEs from Section \ref{currentmethods} satisfy the Assumption.
	\begin{assumption}
		\label{assumptions} There is a constant $C$ such that
		\begin{enumerate}
			\item[(i)] $\beta$ is globally Lipschitz, i.e.$\lVert \beta(x) - \beta(y) \rVert \le C \lVert x - y\rVert$.
			\item[(ii)] $\beta$ grows at most linearly, i.e. $\lVert \beta(x) \rVert \le C(1 + \lVert x \rVert)$.
			\item[(iii)] $X_t$ has a density $\pi_t \in C^1$ for every $t > 0$ and $\int_{t_0}^1 \int_{
				\lVert x \rVert < R} |\pi_t(x)|^2 + \lVert \nabla_x \pi_t(x) \rVert^2 \mathd x \mathd t < \infty$ for any $R > 0$ and $0 < t_0 \le T$.	
		\end{enumerate} 
		Furthermore, for each $S \in (0, T)$ and all $x,y$ for which $\lVert x \rVert, \lVert y \rVert \le N$ there is a constant $C_{S, N}$ such that
		\begin{enumerate}
			\item[(iv)] $\nabla \log \pi_t$ is locally Lipschitz, $\lVert \nabla \log \pi_t(x) - \nabla \log \pi_t(y) \rVert \le C_{S,N} \lVert x - y \rVert$ for all $t \in [S, T]$.
		\end{enumerate}
	\end{assumption}
	Conditions $(i)$-$(iii)$ are technical conditions on the forward SDE. They ensure that if we run a solution to the forward SDE, $X_t$, backwards in time, then $X^R_t := X_{T-t}$ will be a solution to the reverse SDE \eqref{reversesde} on $[0,T)$. The last condition then ensures that the solutions to the reverse SDE are unique, therefore we will be able to transmit the properties of $X^R$ to any other solution $Y_t$ of \eqref{reversesde}. The following result shows that Assumption \ref{assumptions} can be expected to hold in practice.  To simplify the calculations we assume that the data manifold $\mathcal{M} = \text{supp}(\mu_\text{data})$ is contained in a ball of diameter $M$. This is a natural assumption for many data sets. Nevertheless, we note that this assumption could be weakened by additional technical effort.
	\begin{lemma}\label{lemma:assumptions_hold}
		Assume that the data manifold $\mathcal{M}$ is contained in a ball of radius $M$. Then all the methods introduced in Section \ref{currentmethods} fulfil Assumption \ref{assumptions}.
	\end{lemma}
	We are now ready to state our first main result,
	\begin{theorem}\label{thm:mu_prior}
		Denote by $\pi_t$ the marginals of the forward SDE started in $\pi_0$. Assume that Assumption \ref{assumptions} holds and that $\mu_\text{prior}$ is absolutely continuous with respect to $\pi_T$. Then the following hold.
		\begin{itemize}
			\item[(i)] Let $Y_t$ be a solution to \eqref{reversesde} on $[0,T)$. The limit $Y_T := \lim_{t\to T} Y_T$ exists almost surely. We refer to its distribution as $\mu_\text{sample}$. The distribution $\mu_\text{sample}$ is absolutely continuous with respect to $\mu_{\text{data}}$. If $\pi_T$ and $\mu_\text{prior}$ are equivalent, then so are $\mu_\text{sample}$ and $\pi_0$.
			\item[(ii)] Furthermore, for any $f$-divergence $D_f$,
			\[
			D_f(\mu_\text{sample} | \pi_0) \le D_f(\mu_\text{prior} | \pi_T) \quad \text{and} \quad D_f(\pi_0 | \mu_\text{sample}) \le D_f(\pi_T | \mu_\text{prior}).
			\]
		\end{itemize}
	\end{theorem}
	Applying the above theorem with $\pi_T = p_T$ tells us something about the equivalence and distance between $\mu_\text{sample}$ and $\mu_\text{data}$, since $\pi_0 = p_0 = \mu_\text{data}$. Applying the theorem with $\hat{p}_t$ tells us something about the generalization capabilities of SGMs.
	
	The measures $\pi_T$ and $\mu_\text{prior}$ are normally both supported on all of $\mathbb{R}^d$ and therefore equivalent. Since Assumption \ref{assumptions} also holds in most situations (see Lemma \ref{lemma:assumptions_hold}), the requirements for Theorem \ref{thm:mu_prior} are satisfied in practice. From item $(i)$ with $\pi_t = \hat{p}_t$ we can conclude that the error we make in the initial conditions is not responsible for the generalization capacities of SGMs. 
	
	The second item then shows that the $f$-divergences between $\mu_\text{sample}$ and $\mu_\text{data}$ are bounded by the $f$-divergences of $\mu_\text{prior}$ to $p_T$. The total variation distance and the KL-Divergence are both special cases of $f$-divergences.
	
	\subsection{Error in the drift}
	Given a forward SDE with marginals $\pi_t$ and an approximation $s(x,t)$ to $\nabla \log \pi_t$,
	we define the reverse SDE for the approximation $s$ as
	\begin{equation}
		\begin{array}{lll}
			\mathd \tilde{Y}_t & = & - \beta (\tilde{Y}_t) \mathd t + \sigma\sigma^T s(\tilde{Y}_t,t) \mathd t +\sigma  \mathd B_t,\\
			\tilde{Y}_0 & \sim & q_0.
		\end{array} \label{reversesdewitherror}
	\end{equation}

	\begin{assumption}\label{ass:drift}
		We assume that the reverse SDE $\tilde{Y}_t$ has a solution on $[0,T)$. For $t < T$, we define the Girsanov weights
		\begin{equation}
		Z_t = \exp\left(\int_0^t \sigma^T (s(\tilde{Y}_t,t) - \nabla \log \pi_t(\tilde{Y}_t))\cdot \mathd B_s - \frac{1}{2} \int_0^t \lVert \sigma^T (s(\tilde{Y}_t,t) - \nabla \log \pi_t(\tilde{Y}_t)) \rVert^2 \mathd s \right)
		\label{equ:girsanov_weights_def}
		\end{equation}
		and assume that the $Z_t$ are a uniformly integrable martingale.
	\end{assumption}
	We shortly discuss this assumption. 
	The assumption that $Z_t$ is a martingale is equivalent to the expectation of $Z_t$ being equal to 1 for all $t$. 
	The assumption that it is uniformly integrable is more technical (see Appendix \ref{sec:uniform_integrability}), but is fulfilled for example if for some $p > 1$, $\mathbb{E}[|Z_t|^p] < \infty$ for all $t$. For example, if the $Z_t$ have bounded variance, then they are uniformly integrable. 
	
	A condition that ensures that $Z_t$ is both, a martingale and uniformly integrable, is given by \emph{Novikov's condition} \citet{novikov1980conditions}. It states, that if
	\begin{equation}
		N_T = \mathbb{E}_{\tilde{Y}}\left[\exp\left(\frac{1}{2} \int_0^T \lVert \sigma^T (s(\tilde{Y}_t,t) - \nabla \log \pi_t(\tilde{Y}_t)) \rVert^2 \mathd s \right)\right] < \infty,
		\label{equ:novikov}
	\end{equation}
	then $Z_t$ is a uniformly integrable martingale.
	
	Using Assumption \ref{ass:drift} we can now state
	\begin{theorem}\label{thm:reverse_sde_approx_score}
		Assume that Assumption \ref{ass:drift} holds. Assume furthermore that Assumption \ref{assumptions} holds with $\nabla \log \pi_t$ replaced by $s(x, t)$. 
		
		Then $\tilde{Y}_T = \lim_{t \to T} \tilde{Y}_t$ is well defined. Moreover, its distribution is equivalent to the distribution of $Y_T$. In particular, if $\|s(x, t) - \nabla \log \hat{p}_t\|$ is bounded, then Assumption \ref{ass:drift} holds, and the SGM has memorized its training data.
	\end{theorem}
	
	Putting Theorem \ref{thm:mu_prior} and Theorem \ref{thm:reverse_sde_approx_score} together, we can conclude that, if both Assumptions hold, $\mu_\text{sample}$ is equivalent to $\pi_0$. Therefore, if Assumption \ref{ass:drift} holds for $\pi_t = p_t$, we have a positive statement and know that $\mu_\text{sample}$ will have the exact same support as $\mu_\text{data}$, i.e. it has learned the data manifold $\mathcal{M}$.
	
	If the Assumption would hold for $\pi_t = \hat{p}_t$, we would however just memorize the training data, see Figure \ref{fig:sphere_samples_both_disturbed} or Figure \ref{fig:cifar_with_drift_error} for a visualization. However, empirically it has been shown that SGMs are able to create novel samples (see, Figure \ref{fig:distance_to_cifar} or \citet{dhariwal2021diffusion}). Therefore, we can deduce that Assumption \ref{ass:drift} is violated in practice. 

	We evaluated N from \eqref{equ:novikov} on CIFAR-10, once for the difference between the  $s_\theta(x,t)$ from \cite{DBLP:conf/iclr/0011SKKEP21} and $\nabla \log \hat{p}_t$, and once by just using a perturbed drift with a constant error, $s(x, t) = \nabla \log \hat{p}_t + \frac{1}{2}(1, 1, \ldots, 1)$, see Figure \ref{fig:novikov}. The reverse SDE using the drift $\nabla \log \hat{p}_t + \frac{1}{2}(1, 1, \ldots, 1)$, is equivalent to $\hat{\mu}_\text{data}$, as we know from Theorem \ref{thm:reverse_sde_approx_score} and have also already observed in Figure \ref{fig:cifar_with_drift_error}. Figure \ref{fig:novikov} confirms that $Z_t$ is indeed a uniformly integrable martingale and therefore fulfils Assumption \ref{ass:drift}.
	
	Corollary \ref{cor:generalization} can be deduced by exchanging the roles of $\tilde{Y}_t$ and $Y_t$. We then get an equivalent condition to Assumption \ref{ass:drift}, which is
	\[
		\tilde{N}_T = \mathbb{E}_{Y}\left[\exp\left(\frac{1}{2} \int_0^T \lVert \sigma^T (s(Y_t,t) - \nabla \log \pi_t(Y_t)) \rVert^2 \mathd s \right)\right] < \infty,
	\]
	where the expectation is now taken over the reverse SDE with the correct drift $\nabla \log \pi_t$ instead of the approximate drift $s_\theta$. However, $Y$ is just the time reversal of $X$, therefore we can also write the above expectation over $X_t$ instead of $Y_t$. If we treat the case where we start $X_t$ in the empirical measure $\hat{\mu}_\text{data}$, $\pi_t$ will be equal to $\pi_t = \hat{p}_t$ by definition and
	\[
		\mathbb{E}_{\hat{X}}\left[\exp\left(\frac{\sigma^2}{2} \int_0^T \lVert \sigma^T (s(\hat{X}_t,t) - \nabla 	\log \hat{p}_t(\hat{X}_t)) \rVert^2 \mathd s \right)\right] < \infty
	\]	
	is a sufficient condition for $\hat{\mu}_\text{data}$ and $\mu_\text{sample}$ having the same support. We have here assumed that Theorem \ref{thm:mu_prior} can be applied. However, this can be assumed in practice, see the discussion after Theorem \ref{thm:mu_prior}.
	
	In future work we believe that further understanding the characteristics of $Z_t$, how they relate to the minimization of $L(\theta)$ and generalization is crucial to the understanding of SGMs and their empirical success. 
	Recent works also study the question on how the neural network architecture and parametrization is related to the boundedness of the output \citet{DBLP:conf/icml/KimSSKM22}. The relationship between these choices and the properties of $Z_t$ is also an interesting research avenue.
	Lastly, the distribution of $Z_t$ is very heavy-tailed. Most of the samples are very small with a few extremely large samples in between. Therefore one needs many samples from $Z_t$ to understand its characteristics. Finding robust estimators for $Z_t$ or its expectation could help their usage in the training or evaluation procedure.
	
	\begin{figure}
		\centering
		\begin{subfigure}[T]{0.49\textwidth}
			\centering
			\includegraphics[width=\textwidth, height=4.5cm]{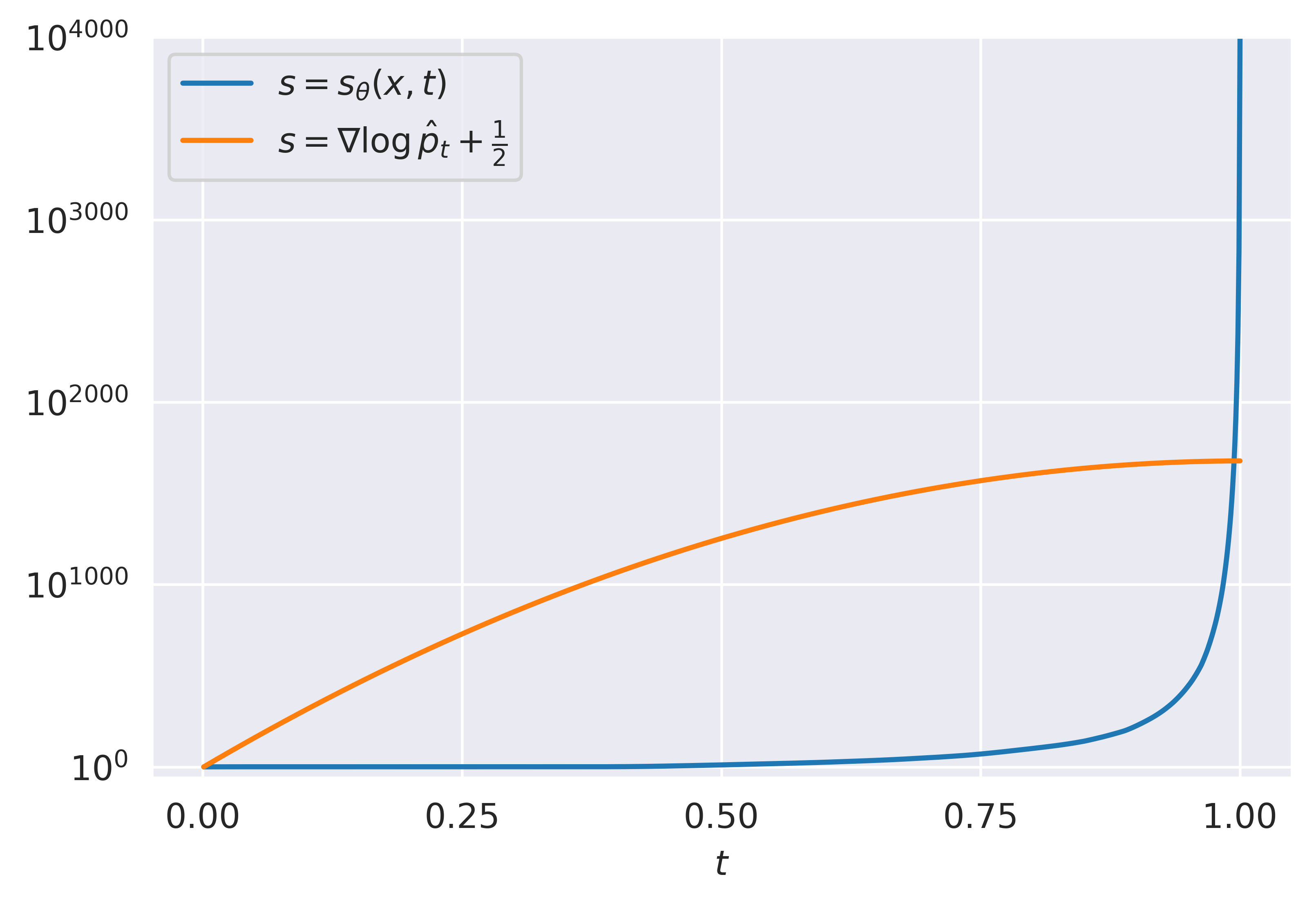}
		\end{subfigure}
		\begin{subfigure}[T]{0.49\textwidth}
			\centering
			\includegraphics[width=\textwidth, height=4.5cm]{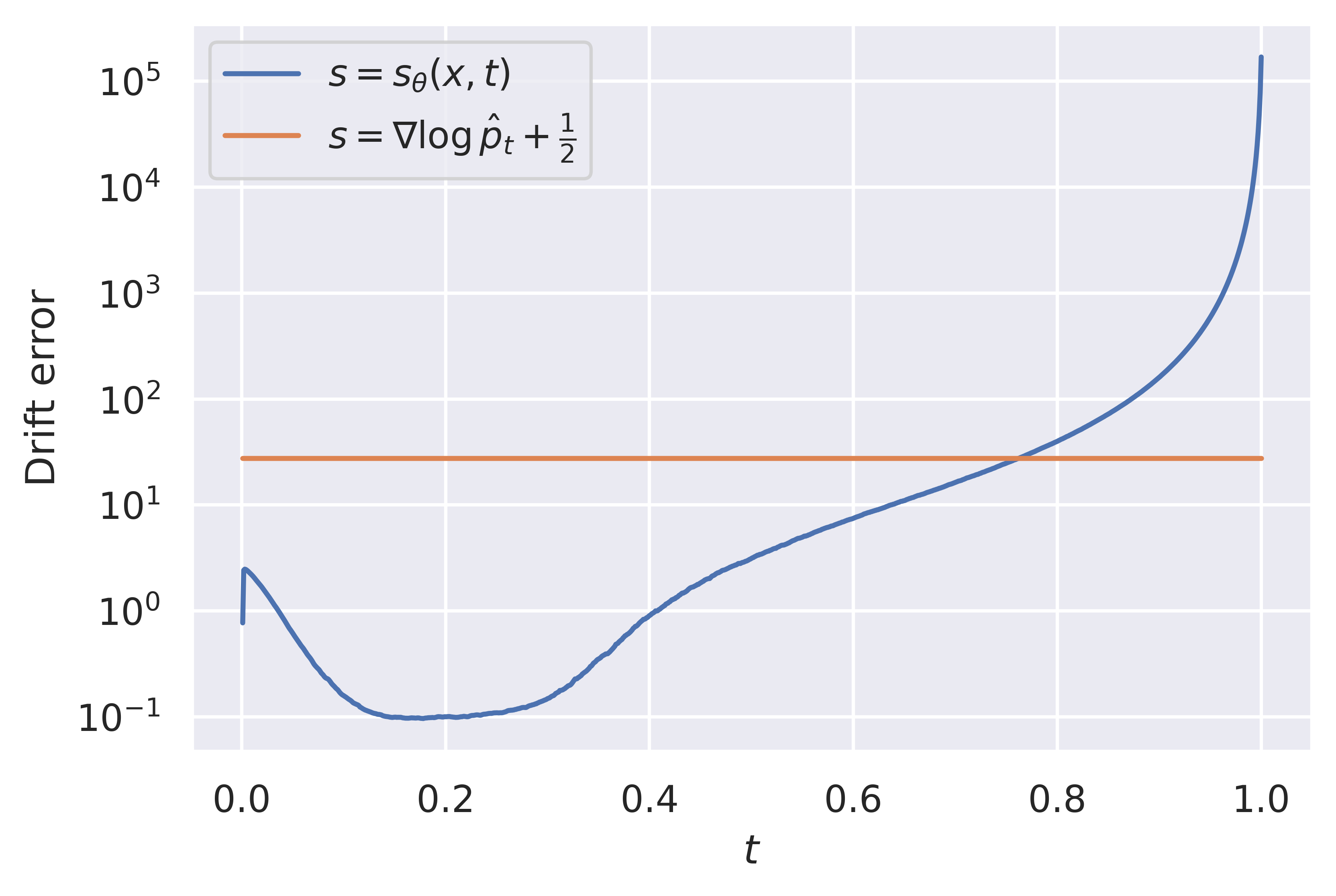}
		\end{subfigure}
		\caption{Left: We simulated the reverse SDE on CIFAR-10, once with the pretrained CIFAR-10 DDPM++ model $s_\theta$ from \cite{DBLP:conf/iclr/0011SKKEP21} and once with a perturbed drift $s(x,t) = \nabla \log \hat{p}_t + \frac{1}{2}(1,1, \ldots, 1)$. We then evaluated the integral \eqref{equ:novikov} numerically for varying $t=T$. For the perturbed drift, the integral does not seem to explode as $t\to 1$, implying that $Z_t$ is a martingale. We see that for the DDPM++ drift, the integral explodes, therefore we can not infer that $Z_t$ is a martingale. We used $N=12 000$ simulations from both of the SDEs to generate this plot. Right: We again ran the two SDEs with the drifts as in the left Figure. This time, we measured the average distance to the empirical drift $\|s(\hat{Y}_t, t) - \nabla \log \hat{p}_t(\hat{Y_t})\|$ along a path of the reverse SDE. We repeated the experiment $N = 2560$ times and plotted the mean distance. For the constant perturbation we also of course get a constant distance. The distance of the true drift to $\nabla \log \hat{p}_t$ is initially very small but explodes as $t \to 1$. From our results we know that this explosion is necessary for the SGM to generalize.}
		\label{fig:novikov}
	\end{figure}


%
%
	
	\section{Drift Explosion under manifold hypothesis}\label{sec:drift_explosion}
	In practice it is often observed that the drift of the reverse SDE explodes as $t \to T$. see for example \citet[Section 3.1]{DBLP:conf/icml/KimSSKM22}. We now show how this observed behaviour is related to the manifold hypothesis.
	
	First, we note that all SDEs in Section \ref{currentmethods} are linear SDEs. Therefore, their transition kernels are Gaussian (\citep[Section 3.7]{pavliotis2014stochastic}):
	\[
	\pi_{t}(X_t = x | X_0 = x_0) = \mathcal{N}(x; m_t(x_0), \Sigma_t).
	\]
	The explicit form of $m_t$ and $\Sigma_t$ differ for each of the SDEs and can be found in Appendix \ref{sec:gaussian_transition_kernels}. We remark that $\Sigma_t$ does not depend on the initial condition $x_0$. The transition kernel above is the distribution of the SDE started in a single point $x_0$. Since we start the SDE in $\mu_\text{data}$ we need to average over $\mu_\text{data}$ to get the marginal at time $t$:
	\begin{equation}\label{equ:p_t_formula}
		p_t(x) = \int_{\mathbb{R}^d} \mathcal{N}(x; m_t(x_0), \Sigma_t) \mu_\text{data}(x_0) \mathd x_0.
	\end{equation}
	We can also compute the additional drift in the reverse SDE (see Appendix \ref{sec:score_form}),
	\begin{equation}\label{equ:drift_form}
		\nabla \log p_t(x) =  \frac{\nabla p_t(x)}{p_t(x)} = \Sigma_t^{-1}(x - \mathbb{E}[m_t(X_0) | X_t = x]).
	\end{equation}
	We now want to evaluate $\nabla \log p_t$ along a typical path of $Y_t$, i.e. we are interested in $\mathbb{E}[\nabla \log p_t(Y_t)]$. The distribution of $Y_t$ however depends on the drift approximation $s_\theta$ we use. For this calculation we will assume that we are able to run the reverse SDE with the true drift $s_\theta(x,t) = p_t(x)$. Then however, since $Y_t$ is then just $X_t$ run backwards, they have the same distributions and we can calculate
	\[
		\mathbb{E}[\|\nabla \log p_t(Y_t)\|] = \mathbb{E}[\|\nabla \log p_t(X_t)\|] = \Sigma_t^{-1}\mathbb{E}[\|X_t - \mathbb{E}[m_t(X_0) | X_t = y]\|].
	\]
	If the manifold $\mathcal{M}$ is not to badly behaved, we can expect that for small $t$ and almost all $x$, $\mathbb{E}[X_0 | X_t = x]$ to be very close to the data manifold $\mathcal{M}$. Especially, $\|X_t - \mathbb{E}[X_0 | X_t = x]\|$ will be larger than $\text{dist}(X_t, \mathcal{M})$ in that case. However, the distribution of $X_t$ can be represented as $m_t(X_0) + \sqrt{\Sigma_t} \xi$, where $\xi$ has a standard normal distribution. For the SDEs we treated in Section \ref{currentmethods}, $m_t(X_t)$ is either equal or very close to $m_t(X_t) = X_t$ for small values of $t$. Finally, if we assume that $\mathcal{M}$ is a subset of relatively low dimension in a high dimensional space, we can expect that with very high probability $\sqrt{\Sigma_t} \xi$ points away from the data manifold. Therefore the distance of $X_t$ to $\mathcal{M}$ can be approximated by $\sqrt{\Sigma_t}\xi$. Putting these approximations together we can calculate
	\[
		\mathbb{E}[\|X_t - \mathbb{E}[m_t(X_0) | X_t = y]\|] \gtrsim \mathbb{E}[\text{dist}(X_t, \mathcal{M})] \gtrsim \|\Sigma_t\|^{1/2} \mathbb{E}[\|\xi\|] \approx \|\Sigma_t\|^{1/2} \sqrt{d},
	\]
	where $d$ is the dimension of the data space in which the samples $x_i$ lie. Therefore we can conclude that
	\[
		\|\nabla \log p_t(Y_t)\| \gtrsim \frac{d^{1/2}}{\|\Sigma_t\|^{1/2}}.
	\]
	For the Brownian motion for example, $\Sigma_t = t$ and therefore the right hand side scales like $\frac{1}{\sqrt{t}}$.
	If $\Sigma_t$ is the covariance of $p_t$, then we can expect $\nabla \log p_t(Y_t)$ to be of order $\frac{1}{\|\Sigma_t\|^{1/2}}$. Furthermore, $\nabla \log p_t(Y_t)$ will point towards the data manifold for small $t$. The drift $\nabla \log p_t(Y_t)$ then acts like a \emph{support matching force}, where the force grows to infinity as $t \to 0$, absorbing all the SDE paths onto the manifold.

	\section{Distance from $p_T$ to $\mu_\text{prior}$}\label{sec:approximate_p_T}
	We have seen in Theorem \ref{thm:mu_prior} that the distance between $\mu_\text{sample}$ and $\mu_\text{data}$ is directly related to the distance between $p_T$ and $\mu_\text{prior}$ if we neglect the errors made in the approximation of the drift. For both, the OU-Process and the CLD, there are plenty of results on the distance of $p_t$ to $\mathcal{N}(0, I_d)$. In general, one can expect this distance to grown exponentially in time $t$.
	
	The Brownian motion however does not converge to a stationary distribution and therefore one has to choose a different $\mu_\text{prior}^T$ for each $T$ to approximate $p_T$. In practice, one normally chooses a normal distribution $p_T = \mathcal{N}(m_t, C_t)$ (see \citep[Appendix C]{DBLP:conf/iclr/0011SKKEP21}). The following Lemma is derives the optimal values for $m_t$ and $C_t$. 
	\begin{lemma}\label{lemma:optimal_mean_cov_bm_model}
		Let $p_T$ be the $T$-time marginal of the Brownian motion process $X_t$ of Section \ref{currentmethods}. The following minimization problem
		\[
		\min_{m,C} KL(p_T~|~\mathcal{N}(m, C))
		\]
		is minimized by $m_T = \mathbb{E}[\mu_\text{data}]$ and $C_T = \mathrm{Cov}[\mu_\text{data}] + T I_d$. If we restrict the covariance to be a multiple of the identity matrix, the problem is solved by choosing $m$ as above and $c$ as $c = \mathbb{E}[\lVert X_t - m \rVert^2] = \text{trace}(C_T)$.
	\end{lemma}
	This result is a slight variation on the well known fact that the $KL$-projection in the second argument matches its moments. We prove it in Appendix \ref{sec:convergence_rates}. The following result shows that the distance between $p_T$ and $\mu_\text{prior}^T$ decreases with time and also gives a rate. It justifies using the Brownian motion and a normal prior distribution for SGMs.
	\begin{lemma}\label{lemma:kl_bm_bound}
		Let $p_t$ be the time $t$-marginal of a Brownian motion with initial condition $\mu_\text{data}$. Denote by $c_i, i=1,\ldots,d,$ the eigenvalues of $\mathrm{Cov}(\mu_\text{data})$. Let $\mu_\text{prior}^T$ be the normal distribution with mean $m_T = \mathbb{E}[\mu_\text{data}]$ and covariance $C_T = \mathrm{Cov}[\mu_\text{data}] + T I_d$. Then
		\[
		KL(p_T | \mu_\text{prior}^T) \le \frac{1}{2}\log\left(\frac{\prod_{i=1}^d (c_i + T)}{T^d}\right).
		\]
	\end{lemma}
	The proof can be found in Appendix \ref{sec:convergence_rates}.
	
	\section{Broader impact}\label{sec:broader-impact}
	The results deepen the understanding of score-based generative models. As such, they can be seen as a step towards improving the quality of generative models. Therefore the possible negative societal impacts are the same ones that apply to generative modelling in general. First, generative models can be used to create synthetic data that is hard to distinguish from real data (for example images or videos), see \citep{mirsky2021creation}. Second, generative models can learn and reproduce biases that are prevalent in the training data (\citep{bias_in_data}). Last, depending on the application, generative models might be used to do creative work that was previously done by humans.

	\section{Conclusion}
	We conducted a theoretical study of some properties of SGMs. We found explicit conditions under which the sample measure $\mu_\text{sample}$ is equivalent to the true data generating distribution $\mu_\text{data}$. Under these conditions we can guarantee, that the SGM generates samples that could also be samples from $\mu_\text{data}$. Furthermore, each sample that can be generated by $\mu_\text{data}$ also has positive probability under $\mu_\text{sample}$, meaning that the full support is covered.
	
	Since one can not actually access the full support of $\mu_\text{data}$, but only a finite number of training examples $\{x_i\}_{i=1}^N$, our results can be applied to find conditions under which the SGM memorizes its training data. We believe that this observation provides a first step towards understanding the generalization capabilities of SGMs.
	
	\section{Funding}
	The author has been partially supported by Deutsche Forschungsgemeinschaft (DFG) - Project-ID 318763901 - SFB1294.

\bibliography{lib}

\begin{thebibliography}{47}
\providecommand{\natexlab}[1]{#1}
\providecommand{\url}[1]{\texttt{#1}}
\expandafter\ifx\csname urlstyle\endcsname\relax
  \providecommand{\doi}[1]{doi: #1}\else
  \providecommand{\doi}{doi: \begingroup \urlstyle{rm}\Url}\fi

\bibitem[Batzolis et~al.(2021)Batzolis, Stanczuk, Sch{\"{o}}nlieb, and
  Etmann]{DBLP:journals/corr/abs-2111-13606}
G.~Batzolis, J.~Stanczuk, C.~Sch{\"{o}}nlieb, and C.~Etmann.
\newblock Conditional image generation with score-based diffusion models.
\newblock \emph{CoRR}, abs/2111.13606, 2021.
\newblock URL \url{https://arxiv.org/abs/2111.13606}.

\bibitem[Bengio et~al.(2013)Bengio, Courville, and
  Vincent]{bengio2013representation}
Y.~Bengio, A.~Courville, and P.~Vincent.
\newblock Representation learning: {A} review and new perspectives.
\newblock \emph{IEEE transactions on pattern analysis and machine
  intelligence}, 35\penalty0 (8):\penalty0 1798--1828, 2013.

\bibitem[Bordes et~al.(2017)Bordes, Honari, and Vincent]{bordes2017learning}
F.~Bordes, S.~Honari, and P.~Vincent.
\newblock Learning to generate samples from noise through infusion training.
\newblock \emph{arXiv preprint arXiv:1703.06975}, 2017.

\bibitem[Cai et~al.(2020)Cai, Yang, Averbuch{-}Elor, Hao, Belongie, Snavely,
  and Hariharan]{DBLP:conf/eccv/CaiYAHBSH20}
R.~Cai, G.~Yang, H.~Averbuch{-}Elor, Z.~Hao, S.~J. Belongie, N.~Snavely, and
  B.~Hariharan.
\newblock Learning gradient fields for shape generation.
\newblock In A.~Vedaldi, H.~Bischof, T.~Brox, and J.~Frahm, editors,
  \emph{Computer Vision - {ECCV} 2020 - 16th European Conference, Glasgow, UK,
  August 23-28, 2020, Proceedings, Part {III}}, volume 12348 of \emph{Lecture
  Notes in Computer Science}, pages 364--381. Springer, 2020.
\newblock \doi{10.1007/978-3-030-58580-8\_22}.
\newblock URL \url{https://doi.org/10.1007/978-3-030-58580-8\_22}.

\bibitem[Chen et~al.(2021)Chen, Zhang, Zen, Weiss, Norouzi, and
  Chan]{DBLP:conf/iclr/ChenZZWNC21}
N.~Chen, Y.~Zhang, H.~Zen, R.~J. Weiss, M.~Norouzi, and W.~Chan.
\newblock Wavegrad: Estimating gradients for waveform generation.
\newblock In \emph{9th International Conference on Learning Representations,
  {ICLR} 2021, Virtual Event, Austria, May 3-7, 2021}. OpenReview.net, 2021.
\newblock URL \url{https://openreview.net/forum?id=NsMLjcFaO8O}.

\bibitem[Costa and Cover(1984)]{costa1984similarity}
M.~Costa and T.~Cover.
\newblock On the similarity of the entropy power inequality and the
  {Brunn-Minkowski} inequality (corresp.).
\newblock \emph{IEEE Transactions on Information Theory}, 30\penalty0
  (6):\penalty0 837--839, 1984.

\bibitem[De~Bortoli et~al.(2021)De~Bortoli, Thornton, Heng, and
  Doucet]{de2021diffusion}
V.~De~Bortoli, J.~Thornton, J.~Heng, and A.~Doucet.
\newblock Diffusion schr{\"o}dinger bridge with applications to score-based
  generative modeling.
\newblock \emph{Advances in Neural Information Processing Systems}, 34, 2021.

\bibitem[Dhariwal and Nichol(2021)]{dhariwal2021diffusion}
P.~Dhariwal and A.~Nichol.
\newblock Diffusion models beat gans on image synthesis.
\newblock \emph{Advances in Neural Information Processing Systems}, 34, 2021.

\bibitem[Dockhorn et~al.(2021)Dockhorn, Vahdat, and
  Kreis]{DBLP:journals/corr/abs-2112-07068}
T.~Dockhorn, A.~Vahdat, and K.~Kreis.
\newblock Score-based generative modeling with critically-damped {Langevin}
  diffusion.
\newblock \emph{CoRR}, abs/2112.07068, 2021.
\newblock URL \url{https://arxiv.org/abs/2112.07068}.

\bibitem[Esser et~al.(2020)Esser, Rombach, and Ommer]{bias_in_data}
P.~Esser, R.~Rombach, and B.~Ommer.
\newblock A note on data biases in generative models.
\newblock In \emph{NeurIPS 2020 Workshop on Machine Learning for Creativity and
  Design}, 2020.
\newblock URL \url{https://arxiv.org/abs/2012.02516}.

\bibitem[Haussmann and Pardoux(1986)]{haussmann1986time}
U.~G. Haussmann and E.~Pardoux.
\newblock Time reversal of diffusions.
\newblock \emph{The Annals of Probability}, pages 1188--1205, 1986.

\bibitem[Ho et~al.(2020)Ho, Jain, and Abbeel]{ho2020denoising}
J.~Ho, A.~Jain, and P.~Abbeel.
\newblock Denoising diffusion probabilistic models.
\newblock \emph{Advances in Neural Information Processing Systems},
  33:\penalty0 6840--6851, 2020.

\bibitem[Ho et~al.(2022)Ho, Saharia, Chan, Fleet, Norouzi, and
  Salimans]{ho2022cascaded}
J.~Ho, C.~Saharia, W.~Chan, D.~J. Fleet, M.~Norouzi, and T.~Salimans.
\newblock Cascaded diffusion models for high fidelity image generation.
\newblock \emph{Journal of Machine Learning Research}, 23\penalty0
  (47):\penalty0 1--33, 2022.

\bibitem[Huang et~al.(2021)Huang, Lim, and Courville]{huang2021variational}
C.-W. Huang, J.~H. Lim, and A.~C. Courville.
\newblock A variational perspective on diffusion-based generative models and
  score matching.
\newblock \emph{Advances in Neural Information Processing Systems}, 34, 2021.

\bibitem[Hyv{\"a}rinen and Dayan(2005)]{hyvarinen2005estimation}
A.~Hyv{\"a}rinen and P.~Dayan.
\newblock Estimation of non-normalized statistical models by score matching.
\newblock \emph{Journal of Machine Learning Research}, 6\penalty0 (4), 2005.

\bibitem[Jeong et~al.(2021)Jeong, Kim, Cheon, Choi, and
  Kim]{DBLP:conf/interspeech/JeongKCCK21}
M.~Jeong, H.~Kim, S.~J. Cheon, B.~J. Choi, and N.~S. Kim.
\newblock Diff-tts: {A} denoising diffusion model for text-to-speech.
\newblock In H.~Hermansky, H.~Cernock{\'{y}}, L.~Burget, L.~Lamel,
  O.~Scharenborg, and P.~Motl{\'{\i}}cek, editors, \emph{Interspeech 2021, 22nd
  Annual Conference of the International Speech Communication Association,
  Brno, Czechia, 30 August - 3 September 2021}, pages 3605--3609. {ISCA}, 2021.
\newblock \doi{10.21437/Interspeech.2021-469}.
\newblock URL \url{https://doi.org/10.21437/Interspeech.2021-469}.

\bibitem[Jolicoeur{-}Martineau et~al.(2021)Jolicoeur{-}Martineau,
  Pich{\'{e}}{-}Taillefer, Mitliagkas, and des
  Combes]{DBLP:conf/iclr/Jolicoeur-Martineau21}
A.~Jolicoeur{-}Martineau, R.~Pich{\'{e}}{-}Taillefer, I.~Mitliagkas, and R.~T.
  des Combes.
\newblock Adversarial score matching and improved sampling for image
  generation.
\newblock In \emph{9th International Conference on Learning Representations,
  {ICLR} 2021, Virtual Event, Austria, May 3-7, 2021}. OpenReview.net, 2021.
\newblock URL \url{https://openreview.net/forum?id=eLfqMl3z3lq}.

\bibitem[Karatzas and Shreve(2012)]{karatzas2012brownian}
I.~Karatzas and S.~Shreve.
\newblock \emph{Brownian motion and stochastic calculus}, volume 113.
\newblock Springer Science \& Business Media, 2012.

\bibitem[Kim et~al.(2022)Kim, Shin, Song, Kang, and
  Moon]{DBLP:conf/icml/KimSSKM22}
D.~Kim, S.~Shin, K.~Song, W.~Kang, and I.~Moon.
\newblock Soft truncation: {A} universal training technique of score-based
  diffusion model for high precision score estimation.
\newblock In K.~Chaudhuri, S.~Jegelka, L.~Song, C.~Szepesv{\'{a}}ri, G.~Niu,
  and S.~Sabato, editors, \emph{International Conference on Machine Learning,
  {ICML} 2022, 17-23 July 2022, Baltimore, Maryland, {USA}}, volume 162 of
  \emph{Proceedings of Machine Learning Research}, pages 11201--11228. {PMLR},
  2022.
\newblock URL \url{https://proceedings.mlr.press/v162/kim22i.html}.

\bibitem[Kingma et~al.(2021)Kingma, Salimans, Poole, and
  Ho]{kingma2021variational}
D.~P. Kingma, T.~Salimans, B.~Poole, and J.~Ho.
\newblock Variational diffusion models.
\newblock \emph{arXiv preprint arXiv:2107.00630}, 2021.

\bibitem[Klenke(2013)]{klenke2013probability}
A.~Klenke.
\newblock \emph{Probability theory: a comprehensive course}.
\newblock Springer Science \& Business Media, 2013.

\bibitem[Kong et~al.(2021)Kong, Ping, Huang, Zhao, and
  Catanzaro]{DBLP:conf/iclr/KongPHZC21}
Z.~Kong, W.~Ping, J.~Huang, K.~Zhao, and B.~Catanzaro.
\newblock Diffwave: {A} versatile diffusion model for audio synthesis.
\newblock In \emph{9th International Conference on Learning Representations,
  {ICLR} 2021, Virtual Event, Austria, May 3-7, 2021}. OpenReview.net, 2021.
\newblock URL \url{https://openreview.net/forum?id=a-xFK8Ymz5J}.

\bibitem[Krizhevsky et~al.(2009)Krizhevsky, Hinton,
  et~al.]{krizhevsky2009learning}
A.~Krizhevsky, G.~Hinton, et~al.
\newblock Learning multiple layers of features from tiny images.
\newblock 2009.

\bibitem[L{\'e}onard(2011)]{leonard2011stochastic}
C.~L{\'e}onard.
\newblock Stochastic derivatives and generalized h-transforms of {Markov}
  processes.
\newblock \emph{arXiv preprint arXiv:1102.3172}, 2011.

\bibitem[Li et~al.(2022)Li, Yang, Chang, Chen, Feng, Xu, Li, and
  Chen]{DBLP:journals/ijon/LiYCCFXLC22}
H.~Li, Y.~Yang, M.~Chang, S.~Chen, H.~Feng, Z.~Xu, Q.~Li, and Y.~Chen.
\newblock Srdiff: Single image super-resolution with diffusion probabilistic
  models.
\newblock \emph{Neurocomputing}, 479:\penalty0 47--59, 2022.
\newblock \doi{10.1016/j.neucom.2022.01.029}.
\newblock URL \url{https://doi.org/10.1016/j.neucom.2022.01.029}.

\bibitem[Liese and Vajda(2006)]{liese2006divergences}
F.~Liese and I.~Vajda.
\newblock On divergences and informations in statistics and information theory.
\newblock \emph{IEEE Transactions on Information Theory}, 52\penalty0
  (10):\penalty0 4394--4412, 2006.

\bibitem[Luo and Hu(2021)]{DBLP:conf/cvpr/LuoH21}
S.~Luo and W.~Hu.
\newblock Diffusion probabilistic models for 3d point cloud generation.
\newblock In \emph{{IEEE} Conference on Computer Vision and Pattern
  Recognition, {CVPR} 2021, virtual, June 19-25, 2021}, pages 2837--2845.
  Computer Vision Foundation / {IEEE}, 2021.
\newblock URL
  \url{https://openaccess.thecvf.com/content/CVPR2021/html/Luo\_Diffusion\_Probabilistic\_Models\_for\_3D\_Point\_Cloud\_Generation\_CVPR\_2021\_paper.html}.

\bibitem[Meng et~al.(2021)Meng, Song, Song, Wu, Zhu, and
  Ermon]{DBLP:journals/corr/abs-2108-01073}
C.~Meng, Y.~Song, J.~Song, J.~Wu, J.~Zhu, and S.~Ermon.
\newblock Sdedit: Image synthesis and editing with stochastic differential
  equations.
\newblock \emph{CoRR}, abs/2108.01073, 2021.
\newblock URL \url{https://arxiv.org/abs/2108.01073}.

\bibitem[Mirsky and Lee(2021)]{mirsky2021creation}
Y.~Mirsky and W.~Lee.
\newblock The creation and detection of deepfakes: {A} survey.
\newblock \emph{ACM Computing Surveys (CSUR)}, 54\penalty0 (1):\penalty0 1--41,
  2021.

\bibitem[Nichol and Dhariwal(2021)]{nichol2021improved}
A.~Q. Nichol and P.~Dhariwal.
\newblock Improved denoising diffusion probabilistic models.
\newblock In \emph{International Conference on Machine Learning}, pages
  8162--8171. PMLR, 2021.

\bibitem[Novikov(1980)]{novikov1980conditions}
A.~A. Novikov.
\newblock On conditions for uniform integrability of continuous non-negative
  martingales.
\newblock \emph{Theory of Probability \& Its Applications}, 24\penalty0
  (4):\penalty0 820--824, 1980.

\bibitem[Pavliotis(2014)]{pavliotis2014stochastic}
G.~A. Pavliotis.
\newblock \emph{Stochastic processes and applications: diffusion processes, the
  {Fokker-Planck} and {Langevin} equations}, volume~60.
\newblock Springer, 2014.

\bibitem[Peluchetti(2021)]{peluchetti2021non}
S.~Peluchetti.
\newblock Non-denoising forward-time diffusions.
\newblock 2021.

\bibitem[Pope et~al.(2021)Pope, Zhu, Abdelkader, Goldblum, and
  Goldstein]{DBLP:conf/iclr/PopeZAGG21}
P.~Pope, C.~Zhu, A.~Abdelkader, M.~Goldblum, and T.~Goldstein.
\newblock The intrinsic dimension of images and its impact on learning.
\newblock In \emph{9th International Conference on Learning Representations,
  {ICLR} 2021, Virtual Event, Austria, May 3-7, 2021}. OpenReview.net, 2021.
\newblock URL \url{https://openreview.net/forum?id=XJk19XzGq2J}.

\bibitem[Popov et~al.(2021)Popov, Vovk, Gogoryan, Sadekova, and
  Kudinov]{DBLP:conf/icml/PopovVGSK21}
V.~Popov, I.~Vovk, V.~Gogoryan, T.~Sadekova, and M.~A. Kudinov.
\newblock Grad-tts: {A} diffusion probabilistic model for text-to-speech.
\newblock In M.~Meila and T.~Zhang, editors, \emph{Proceedings of the 38th
  International Conference on Machine Learning, {ICML} 2021, 18-24 July 2021,
  Virtual Event}, volume 139 of \emph{Proceedings of Machine Learning
  Research}, pages 8599--8608. {PMLR}, 2021.
\newblock URL \url{http://proceedings.mlr.press/v139/popov21a.html}.

\bibitem[Rioul(2010)]{rioul2010information}
O.~Rioul.
\newblock Information theoretic proofs of entropy power inequalities.
\newblock \emph{IEEE Transactions on Information Theory}, 57\penalty0
  (1):\penalty0 33--55, 2010.

\bibitem[Saharia et~al.(2021)Saharia, Ho, Chan, Salimans, Fleet, and
  Norouzi]{DBLP:journals/corr/abs-2104-07636}
C.~Saharia, J.~Ho, W.~Chan, T.~Salimans, D.~J. Fleet, and M.~Norouzi.
\newblock Image super-resolution via iterative refinement.
\newblock \emph{CoRR}, abs/2104.07636, 2021.
\newblock URL \url{https://arxiv.org/abs/2104.07636}.

\bibitem[Sasaki et~al.(2021)Sasaki, Willcocks, and
  Breckon]{DBLP:journals/corr/abs-2104-05358}
H.~Sasaki, C.~G. Willcocks, and T.~P. Breckon.
\newblock {UNIT-DDPM:} {UN}paired image translation with denoising diffusion
  probabilistic models.
\newblock \emph{CoRR}, abs/2104.05358, 2021.
\newblock URL \url{https://arxiv.org/abs/2104.05358}.

\bibitem[Sohl-Dickstein et~al.(2015)Sohl-Dickstein, Weiss, Maheswaranathan, and
  Ganguli]{sohl2015deep}
J.~Sohl-Dickstein, E.~Weiss, N.~Maheswaranathan, and S.~Ganguli.
\newblock Deep unsupervised learning using nonequilibrium thermodynamics.
\newblock In \emph{International Conference on Machine Learning}, pages
  2256--2265. PMLR, 2015.

\bibitem[Song and Ermon(2019)]{DBLP:conf/nips/SongE19}
Y.~Song and S.~Ermon.
\newblock Generative modeling by estimating gradients of the data distribution.
\newblock In H.~M. Wallach, H.~Larochelle, A.~Beygelzimer,
  F.~d'Alch{\'{e}}{-}Buc, E.~B. Fox, and R.~Garnett, editors, \emph{Advances in
  Neural Information Processing Systems 32: Annual Conference on Neural
  Information Processing Systems 2019, NeurIPS 2019, December 8-14, 2019,
  Vancouver, BC, Canada}, pages 11895--11907, 2019.
\newblock URL
  \url{https://proceedings.neurips.cc/paper/2019/hash/3001ef257407d5a371a96dcd947c7d93-Abstract.html}.

\bibitem[Song and Ermon(2020)]{song2020improved}
Y.~Song and S.~Ermon.
\newblock Improved techniques for training score-based generative models.
\newblock \emph{Advances in neural information processing systems},
  33:\penalty0 12438--12448, 2020.

\bibitem[Song et~al.(2020)Song, Garg, Shi, and Ermon]{song2020sliced}
Y.~Song, S.~Garg, J.~Shi, and S.~Ermon.
\newblock Sliced score matching: {A} scalable approach to density and score
  estimation.
\newblock In \emph{Uncertainty in Artificial Intelligence}, pages 574--584.
  PMLR, 2020.

\bibitem[Song et~al.(2021{\natexlab{a}})Song, Durkan, Murray, and
  Ermon]{song2021maximum}
Y.~Song, C.~Durkan, I.~Murray, and S.~Ermon.
\newblock Maximum likelihood training of score-based diffusion models.
\newblock \emph{Advances in Neural Information Processing Systems},
  34:\penalty0 1415--1428, 2021{\natexlab{a}}.

\bibitem[Song et~al.(2021{\natexlab{b}})Song, Sohl{-}Dickstein, Kingma, Kumar,
  Ermon, and Poole]{DBLP:conf/iclr/0011SKKEP21}
Y.~Song, J.~Sohl{-}Dickstein, D.~P. Kingma, A.~Kumar, S.~Ermon, and B.~Poole.
\newblock Score-based generative modeling through stochastic differential
  equations.
\newblock In \emph{9th International Conference on Learning Representations,
  {ICLR} 2021, Virtual Event, Austria, May 3-7, 2021}. OpenReview.net,
  2021{\natexlab{b}}.
\newblock URL \url{https://openreview.net/forum?id=PxTIG12RRHS}.

\bibitem[Vahdat et~al.(2021)Vahdat, Kreis, and Kautz]{vahdat2021score}
A.~Vahdat, K.~Kreis, and J.~Kautz.
\newblock Score-based generative modeling in latent space.
\newblock \emph{Advances in Neural Information Processing Systems}, 34, 2021.

\bibitem[Vincent(2011)]{vincent2011connection}
P.~Vincent.
\newblock A connection between score matching and denoising autoencoders.
\newblock \emph{Neural computation}, 23\penalty0 (7):\penalty0 1661--1674,
  2011.

\bibitem[Zhou et~al.(2021)Zhou, Du, and Wu]{DBLP:conf/iccv/ZhouD021}
L.~Zhou, Y.~Du, and J.~Wu.
\newblock 3d shape generation and completion through point-voxel diffusion.
\newblock In \emph{2021 {IEEE/CVF} International Conference on Computer Vision,
  {ICCV} 2021, Montreal, QC, Canada, October 10-17, 2021}, pages 5806--5815.
  {IEEE}, 2021.
\newblock \doi{10.1109/ICCV48922.2021.00577}.
\newblock URL \url{https://doi.org/10.1109/ICCV48922.2021.00577}.

\end{thebibliography}

\section*{Checklist}

\begin{enumerate}

\item For all authors...
\begin{enumerate}
  \item Do the main claims made in the abstract and introduction accurately reflect the paper's contributions and scope?
    \answerYes{}
  \item Did you describe the limitations of your work?
    \answerYes{The assumptions are followed by a discussion of their strength.}
  \item Did you discuss any potential negative societal impacts of your work?
    \answerYes{See Section~\ref{sec:broader-impact}.}
  \item Have you read the ethics review guidelines and ensured that your paper conforms to them?
    \answerYes{}
\end{enumerate}

\item If you are including theoretical results...
\begin{enumerate}
  \item Did you state the full set of assumptions of all theoretical results?
    \answerYes{}
        \item Did you include complete proofs of all theoretical results?
    \answerYes{}
\end{enumerate}

\item If you ran experiments...
\begin{enumerate}
  \item Did you include the code, data, and instructions needed to reproduce the main experimental results (either in the supplemental material or as a URL)?
    \answerYes{}
  \item Did you specify all the training details (e.g., data splits, hyperparameters, how they were chosen)?
    \answerYes{See Appendix~\ref{sec:numerics}.}
        \item Did you report error bars (e.g., with respect to the random seed after running experiments multiple times)?
    \answerNo{There are no error to report, we just have small illustrative numerical experiments.}
        \item Did you include the total amount of compute and the type of resources used (e.g., type of GPUs, internal cluster, or cloud provider)?
    \answerYes{See Appendix~\ref{sec:numerics}.}
\end{enumerate}

\item If you are using existing assets (e.g., code, data, models) or curating/releasing new assets...
\begin{enumerate}
  \item If your work uses existing assets, did you cite the creators?
    \answerNA{We did not use any existing assets.}
  \item Did you mention the license of the assets?
    \answerNA{}
  \item Did you include any new assets either in the supplemental material or as a URL?
    \answerNA{}
  \item Did you discuss whether and how consent was obtained from people whose data you're using/curating?
    \answerNA{}
  \item Did you discuss whether the data you are using/curating contains personally identifiable information or offensive content?
    \answerNA{}
\end{enumerate}

\item If you used crowdsourcing or conducted research with human subjects...
\begin{enumerate}
  \item Did you include the full text of instructions given to participants and screenshots, if applicable?
    \answerNA{We did not use crowdsourcing or conduct research with human subjects.}
  \item Did you describe any potential participant risks, with links to Institutional Review Board (IRB) approvals, if applicable?
    \answerNA{}
  \item Did you include the estimated hourly wage paid to participants and the total amount spent on participant compensation?
    \answerNA{}
\end{enumerate}

\end{enumerate}


\newpage
\appendix
\section{Stochastic prerequisites}\label{sec:prerequisites}
In this section we give a formal introduction to some of the concepts used in this work. For a more rigorous treatment, see for example \citep{klenke2013probability} or \citep{karatzas2012brownian}.
\subsection{Equivalence of measures / Girsanov Theorem}\label{sec:equivalence_of_measures}
First we define absolute continuity of measures. Let $\mu$ and $\nu$ be two measures on $(\Omega, \mathcal{F})$, where $\mathcal{F}$ is a $\sigma$-algebra. 
\begin{definition}
	We say that $\mu$ is absolutely continuous with respect to $\nu$ if $\mu(A) = 0$ for any $A \in \mathcal{F}$ such that $\nu(A) = 0$. We also denote this by $\mu \ll \nu$.
\end{definition}
Two measures $\mu$ and $\nu$ are equivalent if $\mu \ll \nu$ and $\nu \ll \mu$. Loosely speaking, we can say that $\mu \ll \nu$ if the support of $\mu$ is contained in the support of $\nu$ and they are equivalent if they share the same support.

The Radon-Nikodym theorem tells us that if $\mu \ll \nu$, then under mild conditions there exists a density $\frac{\mathd \mu}{\mathd \nu}: \Omega \to \mathbb{R}$ such that $\mu(A) = \int_{A} \frac{\mathd \mu}{\mathd \nu}(\omega) \mathd \nu(\omega)$. Therefore, we can obtain $\mu$ through a reweighting of $\nu$. One specific instance of this is the Girsanov Theorem. Assume we are given the solutions to two SDEs in $\mathbb{R}^d$,
\begin{align}\label{equ:appendix_Yt}
		\mathd Y_t = b(t, Y_t) \mathd t + \sigma(t, Y_t) \mathd W_t
\end{align}
and
\begin{align}
	\mathd \tilde{Y}_t = b(t, \tilde{Y}_t) \mathd t + \sigma(t, \tilde{Y}_t) e(t, \tilde{Y}_t) \mathd t + \sigma(t, \tilde{Y}_t) \mathd B_t.
\end{align}
Both of these induce a measure on the space of continuous functions $\Omega = C([0, T], \mathbb{R}^d)$. We denote them by $\mathbb{P}$ and $\tilde{\mathbb{P}}$ respectively. Then the Girsanov Theorem equips us with conditions under which the measures $\mathbb{P}$ and $\tilde{\mathbb{P}}$ are equivalent. Furthermore, in case of equivalence we get a formula for the density of $\tilde{\mathbb{P}}$ with respect to $\mathbb{P}$. The relative density is given as
\[
	Z_T = \exp\left( \int_0^T e(s, Y_s) \mathd W_s - \frac{1}{2} \int_0^T \lVert e(s, Y_s) \rVert^2 \mathd s \right).
\]
For a full statement of the Girsanov Theorem and under which conditions it holds, see \citep[Section 3.5]{karatzas2012brownian}.
\subsection{Uniform integrability}\label{sec:uniform_integrability}
Since we are treating the case where the drift explodes as $t \to T$ we end up with densities 
\begin{equation}\label{equ:girsanov_weights}
Z_t = \exp\left( \int_0^t e(s, Y_s) \mathd W_s - \frac{1}{2} \int_0^t \lVert e(s, Y_s) \rVert^2 \mathd s \right).
\end{equation}
on $C([0,t], \mathbb{R}^d)$, but not with a density on $C([0,T], \mathbb{R}^d)$. Uniform integrability is exactly the condition one needs to extend these local densities.
\begin{definition}
	A family $\{X_\alpha\}$ of random variables is called uniformly integrable if
	\begin{equation*}
		\sup_{\alpha}~\mathbb{E}\left[|X_\alpha|~1_{\{|X_\alpha| > s\}}\right] \to 0
	\end{equation*}
	as $s \to \infty$.
\end{definition}
In the proof of Theorem \ref{thm:reverse_sde_approx_score} we implicitly use the following two results which we here state as a lemma. The filtration $\mathcal{F}_t$ is defined as in the proof of Theorem \ref{thm:reverse_sde_approx_score}.
\begin{lemma}
	Assume the $Z_t$ in \eqref{equ:girsanov_weights} form a uniformly integrable martingale on $[0,T)$. Then,
	\begin{itemize}
	\item the limit $\lim_{t \to T} Z_t$ exists in $L^1$. We denote this limit by $Z$. 
	\item Furthermore, $\tilde{\mathbb{P}}$ is absolutely continuous with respect to $\mathbb{P}$ on $\mathcal{F} = \sigma(\cup_{t < T} \mathcal{F}_t)$ with density $Z$.
	\end{itemize} 
\end{lemma}
\begin{proof}
	Both of these results are standard. The first one can for example be found in \citep[Section 1.3.B]{karatzas2012brownian}. For the second one we compute that for any $A \in \mathcal{F}_s$,
	\[
		\mathbb{E}_{\mathbb{P}}[1_{A} Z] 
		= \mathbb{E}_{\mathbb{P}}[1_A \lim_{t \to T} Z_t]
		= \lim_{t \to T} \mathbb{E}_{\mathbb{P}}[1_A Z_t] 
		=  \mathbb{E}_{\mathbb{P}}[1_A Z_s] 
		= \mathbb{E}_{\tilde{\mathbb{P}}}[1_A],
	\]
	where we used $L^1$ convergence in the second equality and the martingale property of $Z_s$ in the third equality. Therefore $Z$ is a density of $\tilde{\mathbb{P}}$ with respect to $\mathbb{P}$ on each $\mathcal{F}_s$ for $s < T$. Therefore $Z$ is also a density of $\tilde{\mathbb{P}}$ with respect to $\mathbb{P}$ on $\mathcal{F}$ which concludes the proof.
\end{proof}

\section{Numerics}\label{sec:numerics}
All numerical experiments can be run on a consumer grade computer within a few minutes. 
\subsection{Figure \ref{fig:forward_backward_problem}}
We first discuss the top left figure of Figure \ref{fig:forward_backward_problem}.
We set $p_0 = \mu_\text{data}$ to a mixture of two Gaussian $\mathcal{N}(-2, \frac{1}{100})$ and $\mathcal{N}(2, \frac{1}{100})$ with weights $w_1 = \frac{1}{3}$ and $w_2 = \frac{2}{3}$ respectively. Then, we draw $N=5~000~000$ samples from $\mu_\text{data}$, denoted by $Y_0^n$, $n=1,\ldots,N$. An Euler-Maruyama discretization of the Brownian motion propagates these samples from time $t=0$ to $t=1$ by
\[
	X_{i+1}^n = X_i^n + \sqrt{dt} Z_i^n,
\]
where $Z_n^i \sim \mathcal{N}(0, 1)$ are i.i.d. random variables, independent of $X_m^j$ for $m \le n$ and $j=1,\ldots, N$. The time index $i$ runs from $0$ to $I=2000$ and $dt$ is set to $dt = \frac{1}{I}$. The initial samples $\{X_0^n\}_{n=1}^N$ are used to create the left line plot of $p_0$ and the final samples $\{X_I^n\}_{n=1}^N$ are used to create the right line plot of $p_1$ using kernel density estimation. The $\{X_i^n\}_{n=1}^N$ are approximate samples from $p_{i/I}$. Therefore, we create histograms using $\{X_i^n\}_{n=1}^N$ to approximate $p_{i/I}$. The height of the histogram bars corresponds to the square root of the colour intensity in the heat map. The horizontal axis in the heat map stands for the time $t$, whereas the vertical axis stands for the position $x$. At location $(t, x)$ we plot an estimate of $\sqrt{p_t(x)}$. We apply the square root since it improves the contrast in areas where $p_t(x)$ is close to $0$ and makes it more visible where $p_t(x) > 0$ to the observer.

For the bottom left figure we show the same plots, just for the reverse SDE \eqref{reversesdewitherror} instead of the forward SDE. Since the initial distribution is a Gaussian mixture we can exactly calculate $p_t$ using
\begin{equation}\label{equ:app_explicit_pt_example}
	p_t(x) = w_1 \mathcal{N}(x ; m_1, s_1^2 + t) + w_2 \mathcal{N}(x ; m_2, s_2^2 + t),
\end{equation}
where we use $\mathcal{N}(x ; m, v)$ for the probability density function of a normal distribution with mean $m$ and variance $v$, evaluated at $x$. With the above expression of $p_t$ one could compute an analytical representation of $\nabla \log p_t$. We use automatic differentiation instead. The reverse SDE \eqref{reversesdewitherror} is simulated with a disturbance $e(x, t) = 1$ and initial condition $q_0 = \mu_\text{prior} = \mathcal{N}(0, 1)$. The Euler-Maruyama method is run with the same step size $dt = \frac{1}{I}$. More precisely, the one step transition kernel of the discretized reverse SDE is
\begin{equation}\label{equ:app_reverse_sde_euler_maruyama}
	Y_{i+1}^n = Y_i^n + dt~(\nabla \log p_{1 - \frac{i}{I}}(Y_i^n) + 1) + \sqrt{dt}~\tilde{Z}_i^n, 
\end{equation}
where $\tilde{Z}_n^i \sim \mathcal{N}(0, 1)$ are i.i.d. random variables, independent of $Y_m^j$ for $m \le n$ and $j=1,\ldots, N$. The plots are created in the same way as for the upper left plot, except that we reverse the time axis to plot $p_t$ and $q_{1-t}$ directly underneath each other.

On the right side we plot the same kernel density estimates already plotted on the left side as $\mu_\text{sample}$ and $\mu_\text{data}$ into the same plot for comparison.
\subsection{Figure \ref{fig:sphere_samples_both_disturbed} and \ref{fig:sphere_both_disturbed}}
Figure \ref{fig:sphere_samples_both_disturbed} is created by setting $\mu_\text{data}$ to be the uniform distribution on $M=9$ equally spaced samples $\{x_i\}_{i=1}^M$ on the unit sphere $\mathcal{S}^1$. This can also be viewed as a Gaussian mixture with $9$ components where each component having mean $x_i$ and variance $0$. Therefore, we can again explicitly calculate $p_t$ for $t > 0$ as in \eqref{equ:app_explicit_pt_example},
\[
	p_t(y) = \frac{1}{M} \sum_{i=1}^M \mathcal{N}(y; x_i, t).
\]
The score $\nabla \log p_t(y)$ is evaluated using automatic differentiation. The reverse SDE \ref{reversesdewitherror} is simulated with $q_0 = \mathcal{N}((x = -1.5, y=0), I_2)$ and $e((x,y),t) = (x=0, y=-1)$. For the numerical simulation, we again use the Euler-Maruyama scheme. We use a step width of $dt = \frac{0.9}{1000} \approx \frac{1}{1000}$ for $t=[0, 0.9]$, $dt = \frac{0.09}{1000} \approx \frac{1}{10~000}$ for $t\in [0.9, 0.99]$ and of $dt = \frac{1}{100 ~ 000}$ for $t\in[0.99, 1]$.
For a simulation of $N = 50~000$ paths of the reverse SDE, we start by drawing $Y_0^n = A_n + Z_n$, where $A_n$ are i.i.d. uniformly distributed on $\{x_i\}_{i=1}^M$ and $Z_n \sim \mathcal{N}(0, 1)$ i.i.d.. The $\{A_n\}$ and $\{Z_n\}$ are also independent from each other. We then propagate the $Y_0^n$ similarly to \eqref{equ:app_reverse_sde_euler_maruyama}, except that we use a different values for $dt$ depending on $t$. This leads to approximate samples $Y_t^i$ from $q_t$. At the displayed times $t$ we plot the function
\[
	h_t(x) =  \sum_{i=1}^N k(x, Y_t^i),
\]
where $k$ is an unnormalized Gaussian kernel with a very small bandwidth parameter,
\[
k(x, y) = \exp(-1000 \lVert x - y \rVert^2).
\]
Normalizing $h_t$ gives us a density estimate of $q_t$. We plot these estimates as heat maps for different values of $t$.

For Figure \ref{fig:sphere_both_disturbed} we follow the same steps as for Figure \ref{fig:sphere_samples_both_disturbed}, except that $\mu_\text{data}$ is set to the uniform distribution over $M=256$ evenly spaced samples from the unit sphere $\mathcal{S}^1$.

\subsection{Figure \ref{fig:cifar_distance_and_drift_error}}
For Figure \ref{fig:cifar_distance_and_drift_error} we used the DDPM++ model from the Github repository for the paper from \cite{DBLP:conf/iclr/0011SKKEP21}. Then we evaluated the true score using \eqref{equ:hat_pi_definition}, where the sum runs through the $N = 50 000$ training examples of CIFAR-10, \citet{krizhevsky2009learning}. A similar experiment using the true marginals $\hat{\mu}_\text{data}$ has also been conducted in \cite{peluchetti2021non}.

\section{Studying the forward Densities}
\subsection{Transition kernels}\label{sec:gaussian_transition_kernels}
The transition kernels $p(z_0, \cdot)$ for the SDEs from Section \ref{currentmethods} are of the from
\[
	q_0(z_0, \cdot) = \mathcal{N}(m_t(z_0), \Sigma_t),
\]
where $m_t$ and $\Sigma_t$ are given in Table \ref{tab:mean_cov_transition_kernels} for the Brownian Motion and the Ornstein-Uhlenbeck process. The form of the transition kernels for the CLD are more involved. They can be found in \citet[Appendix B.1]{DBLP:journals/corr/abs-2112-07068}.

\begin{lemma}
	The marginal densities $p_t(x)$ of the SDEs treated in Section \ref{currentmethods} depend smoothly on $x$ and $t$.
\end{lemma}
\begin{proof}
	One can combine the form of $m_t$ and $\Sigma_t$ and the explicit representation of $p_t$ in \eqref{equ:p_t_formula} to see this. More generally, for the Brownian motion and the OU-Process this is a result of the H\"ormander theorem. For the CLD it is a result of hypocoercivity. 
\end{proof}

\begin{table}
\centering
\begin{tabular}{ c c c c } 
		& & & \\[-2ex]
		& $m_t(z_0)$ & $\Sigma_t$ & \\[0.5ex]
		\hline & & &\\[-1.5ex]
		Brownian Motion  & $z_0$ & $t I_d$ & \\[1ex]
		OU-Process & $\exp(-t)z_0$ & $(1 - \exp(-2t)) I_d$ & \\[1.5ex]
	\end{tabular}
\caption{\label{tab:mean_cov_transition_kernels}
		The mean and covariance of the Gaussian transition kernels of the Brownian Motion and the Ornstein-Uhlenbeck process SDEs.
	}
\end{table}

\subsection{Form of the drift}\label{sec:score_form}
We now prove that we can represent the drift as in \eqref{equ:drift_form}.
\begin{lemma}\label{lemma:score_form}
	Assume that $p_t$ has the form \eqref{equ:p_t_formula}, i.e.
	\begin{equation*}
		p_t(z) = \int_{\mathbb{R}^d} \mathcal{N}(z; m_t(z_0), \Sigma_t) \mu_\text{data}(z_0) \mathd z_0.
	\end{equation*}
	Then \eqref{equ:drift_form} holds true, i.e.
	\begin{equation*}
		\nabla \log p_t =  \frac{\nabla p_t(z)}{p_z(z)} = \Sigma_t^{-1}(z - \mathbb{E}[m_t(Z_0) | Z_t = z]).
	\end{equation*}
\end{lemma}
\begin{proof}
	\begin{align*}
		&\nabla p_t(z) \\
		=&~ \frac{1}{\sqrt{\det(2\pi\Sigma_t)}} \int_{\mathbb{R}^d}\nabla \exp\left(-\frac{1}{2}(z - m_t(z_0))\Sigma_t^{-1}(z-m_t(z_0))\right) \mu_\text{data}(z_0) \mathd z_0 \\
		=&~ \frac{1}{\sqrt{\det(2\pi\Sigma_t)}} \int_{\mathbb{R}^d} \Sigma_t^{-1}(z-m_t(z_0)) \exp\left(-\frac{1}{2}(z - m_t(z_0))\Sigma_t^{-1}(z-m_t(z_0))\right) \mu_\text{data}(z_0) \mathd z_0 \\
		=&~ \Sigma_t^{-1}z p_t(z) \\
		&- \frac{1}{\sqrt{\det(2\pi\Sigma_t)}} \int_{\mathbb{R}^d} \Sigma_t^{-1}m_t(z_0) \exp\left(-\frac{1}{2}(z - m_t(z_0))\Sigma_t^{-1}(z-m_t(z_0))\right) \mu_\text{data}(z_0) \mathd z_0.
	\end{align*}
	If we now divide everything by $p_t$ it cancels in the first summand. In the second summand we get the formula for the conditional expectation (see, for example \citep[Section 8.2]{klenke2013probability}).
\end{proof}

\section{Reverse SDEs: The general case}\label{sec:sdes_formal}
One can also treat more general forward SDEs than we did in Section \ref{intro}. This
leads to a more complicated form of the reverse SDE. Our Theorems do not use the specific
form of the forward SDEs and therefore also hold in the general case. We denote
the forward SDE by
\begin{equation}
	\begin{array}{lll}
		\mathd X_t & = & \beta (t, X_t) \mathd t + \sigma (t, X_t) \mathd W_t,\\
		X_0 & \sim & \mu_\text{data},
	\end{array} \text{} \label{forwardsdeformal}
\end{equation}
where $\mu_\text{data}$ is supported on $\mathcal{M} \subset \mathbb{R}^d$ and $W_t$ is a $\mathbb{R}^r$ valued Brownian motion. The drift $b$ maps from $\mathbb{R} \times \mathbb{R}^d$ to $\mathbb{R}^d$. The dispersion coefficient $\sigma$ maps from $\mathbb{R} \times \mathbb{R}^d$ to the $d \times r$-matrices. The time-reversed process $Y_{t} := X_{T - t}$ is then a solution to
\begin{equation}
	\begin{array}{lll}
		\text{dY}_t & = &  b (t, Y_t) \mathd t + \sigma (T -
		t, Y_t) \mathd B_t,\\
		Y_0 & \sim & q_0,
	\end{array} \label{reversesdeformal}
\end{equation}
with
\[ b_i (t, y) = - \beta (T - t, y) + \frac{\sum_j \nabla_j (a_{\text{ij}} (T - t, y) p_{T - t}
	(y))}{p_{T - t} (y)}, \quad a (t, y) = \sigma (t, y) \sigma (t, y)^T, \]
and $q_0 = p_T$, see \citet{haussmann1986time}.
This simplifies to the case discussed in Section \ref{intro} if $r=d$, $\sigma$ is is a multiple of the identity matrix and $\beta$ is independent of the time $t$. In Assumption \ref{assumptions} we treat the case where the SDEs are of the form written-out in Section \ref{intro}. The items $(i)-(iii)$ need to be replaced by their more general counterpart as found in \citet[Section 2]{haussmann1986time}. In the last item $(iv)$, $\nabla \log p_t$ needs to be replaced by $\frac{\sum_j \nabla_j (a_{\text{ij}} (T - t, y) p_{T - t}
	(y))}{p_{T - t} (y)}$.

\section{Proofs}

\subsection{Proofs of the theorems}\label{formalproof}
We now give proofs of our main results and briefly summarize the key steps in an intuitive way. In our study we would like to include the case when $\mu_\text{data}$ is degenerate and supported on a low-dimensional substructure $\mathcal{M}$. As we have seen in Section \ref{sec:assumptions_hold}, this can lead to an exploding drift in the reverse SDE as $t \to T$. Nevertheless, in order to understand the properties of $\mu_\text{data}$, it is crucial to study the properties of solutions to the reverse SDE at time $t=T$. This is where the main mathematical difficulties come from. The proofs are mostly independent of the specific form of the forward SDE and hold for more general forward/backward SDEs than those stated in Section \ref{intro}, see Appendix \ref{sec:sdes_formal}.
\subsubsection{Theorem 1}
We now proceed with proving Theorem \ref{thm:mu_prior}.
\begin{proof}
	Let $P$ be the measure on $\Omega = C([0,T], \mathbb{R}^n)$ induced by the forward SDE \eqref{diffusingprocesscont} started in $p_0 = \mu_\text{data}$. $P$ has marginals $p_t$. Denote by $X_t$ the canonical projections $X_t(\omega) = \omega(t)$ for $\omega \in \Omega$. We define $Q$ through
	\[
	\frac{\mathd Q}{\mathd P}(\omega) = \frac{\mathd \mu_\text{prior}}{\mathd \pi_T}(\omega(T)).
	\]
	By the data processing inequality we obtain (see \citep[Theorem 14]{liese2006divergences}),
	\begin{equation}\label{eq:kl_inequality}
		KL(q_T | \mu_\text{data}) \le KL(Q | P) = KL(\mu_\text{prior} | \pi_T).
	\end{equation}
	It remains to prove that by running $Q$ backwards we obtain a solution to \eqref{reversesde} started in $\mu_\text{prior}$. We denote the generator of the reverse SDE \eqref{reversesde} by $\mathcal{L}$.
	Denote by $Q^R$ and $P^R$ the time reversals of $Q$ and $P$. Our assumption are such that $P^R$ is a Markov process solving the martingale problem for $\mathcal{L}$ (see \citep[Theorem 2.1]{haussmann1986time}).
	A short calculation shows that $Q^R$ is still Markov (see, for example \citep[Proposition 4.2]{leonard2011stochastic}). 
	Furthermore for $f\in C^\infty_c(\mathbb{R}^n)$,
	\begin{align*}
		&\mathbb{E}_{Q^R}\left[f(X_t) - f(X_s) - \int_s^t \mathcal{L}f(X_r) \mathd r | X_s\right] \\
		=&~
		\frac{\mathbb{E}_{P^R}\left[\left(f(X_t) - f(X_s) - \int_s^t \mathcal{L}f(X_r) \mathd r\right)\frac{\mathd \mu_\text{prior}}{\mathd p_T}(X_0)| X_s\right]}
		{\mathbb{E}_{P^R}\left[\frac{\mathd \mu_\text{prior}}{\mathd p_T}(X_0)| X_s\right]} \\
		=&~ \frac{
			\mathbb{E}_{P^R}\left[\left(f(X_t) - f(X_s) - \int_s^t \mathcal{L}f(X_r) \mathd r\right)|X_s\right]
			\mathbb{E}_{P^R}\left[\frac{\mathd \mu_\text{prior}}{\mathd p_T}(X_0)|X_s\right]
		}
		{\mathbb{E}_{P^R}\left[\frac{\mathd \mu_\text{prior}}{\mathd p_T}(X_0)| X_s\right]} \\
		=&~ \mathbb{E}_{P^R}\left[\left(f(X_t) - f(X_s) - \int_s^t \mathcal{L}f(X_r) \mathd r\right)|X_s\right] = 0.
	\end{align*}
	In the second equality we used the Markov property of $P^R$. In the last one we used that $P^R$ solves the martingale problem for $\mathcal{L}$. Therefore also $Q^R$ solves the martingale problem for $\mathcal{L}$. Denote by $Y_t$ a solution to \eqref{reversesde} on $[0,T)$. 
	Since solutions to \eqref{reversesde} are unique in law on $[0,S]$ for $S<T$ (see \citep[Section 5.2]{karatzas2012brownian}) and the solutions are continuous, the law of $Y$ is equal to $Q^R$ on $[0,T)$. But the paths of $Q^R$ are continuous on $[0,T]$. Therefore, $Y$ can be extended to $[0,T]$, i.e. the limit $Y_T := \lim_{t\to T} Y_t$ exists almost surely and its distribution is equal to the $T$-time marginal distribution of $Q^R$, which is the $0$-time marginal of $Q$. We denote the marginals of $Q$ and $P$ by $Q_t$ and $P_t$.
	Since $Q$ is absolutely continuous with respect to $P$, $Q_0 = \mu_\text{sample}$ is absolutely continuous with respect to $P_0 = \mu_\text{data}$. Analogously, if $\mu_\text{prior}$ and $p_T$ are equivalent, then so are $P$ and $Q$ and therefore $P_0$ and $Q_0$. This proves $(i)$.
	
	$(ii)$ is a consequence of the data processing inequality for $f$-divergences (\citep[Theorem 14]{liese2006divergences}), analogous to \eqref{eq:kl_inequality}.
\end{proof}

The main idea of this proof is that we look at the forward SDE for $X_t$ first. It induces a distribution $\mathbb{P}$ over all continuous paths in $\Omega =C([0,T], \mathbb{R}^d)$. If we reverse the time direction of this distribution on $\Omega$, we get a solution to the reverse SDE, started in $p_T$. This reverse solution is well behaved as $t \to T$, since $Y_T = X_0$. The solution for a different initial condition $\mu_\text{prior} \not= p_T$ is obtained by reweighting $\mathbb{P}$. This does not change the qualitative behaviour of $Y_T$, which still exists and is well defined. We then use a uniqueness result to see that any solution of \eqref{reversesdeformal} inherits these properties.

\subsubsection{Theorem 2}

\begin{proof}
	Denote the space $C([0,T), \mathbb{R}^d)$ by $\Omega$. 
	We also define the natural filtration $\mathcal{F}_t = \sigma(x(s) | s \le t)$.
	We denote the distribution of $Y$ on $\Omega$ by $\mathbb{P}$. We define $\tilde{\mathbb{P}}$ by reweighting $\mathbb{P}$ with $Z_t$ on $\mathcal{F}_t$. By Girsanov theorem (see \citep[Section 3.5]{karatzas2012brownian}) we know that the canonical process under $\tilde{\mathbb{P}}$ is a solution to \eqref{reversesdewitherror} on $[0,T)$. Since $Z_t$ is uniformly integrable, its limit $Z:= \lim_{t \to T} Z_t$ exists in $L^1$. Furthermore $\tilde{\mathbb{P}}$ is absolutely continuous with respect to $\mathbb{P}$ on $\mathcal{F} = \sigma(\cup_{t < T} \mathcal{F}_t) = \sigma(x(t) | t < T)$ with density $Z$. We define $x(T) := \lim_{t \to T} x(t)$. Then the event
	\[
	\mathcal{A} = \{x(T) := \lim_{t\to T} x(t) \text{ exists and } x(T) \in \mathcal{M}\}
	\]
	has probability $1$ under $\mathbb{P}$ (see Theorem \ref{thm:mu_prior}) and therefore also under $\tilde{\mathbb{P}}$. Furthermore, $x(T)$ is measurable with respect to $\mathcal{F}$.
	Therefore the distributions of $x(T)$ under $\mathbb{P}$ and $\tilde{\mathbb{P}}$ are equivalent. The canonical process under $\tilde{\mathbb{P}}$ is therefore a solution of \eqref{reversesdewitherror}, with the property that its time $T$-marginal is well defined and equivalent to the time $T$-marginal of \eqref{reversesde}. We can use uniqueness in law on $[0,S)$ for any $S < T$ and extend it to $[0,T]$ as in the proof of Theorem \ref{thm:mu_prior}. This shows that every solution to \eqref{reversesdewitherror} has the desired properties.
	
	Finally we show that if $e$ is bounded, it fulfils Assumptions \ref{ass:drift}. We define by $H_t = \int_0^t \lVert e(s, \hat{Y}_s) \rVert \mathd s$. Then there is a Brownian motion $W_t$ such that we can write $Z_t$ as
	\[
	Z_t = \exp\left(W_{H_t} - \frac{1}{2}H_t\right).
	\]
	Since $e$ is bounded by $M$, $H_t$ is bounded by $TM$. In particular, one can view $Z_t = \mathbb{E}[Z_{TM} | \mathcal{F}_{H_t}]$. Therefore $Z_t$ is uniformly integrable since it can be viewed as a family of conditional expectations. 
\end{proof}

Here we essentially applied the Girsanov Theorem on $[0,T)$. Using the uniform integrability of the Girsanov weights $Z_t$, we are able to extend it to $[0,T]$. Therefore, we can infer that the distribution of $Y$ and $\hat{Y}$ are actually equivalent on the whole path space $C([0,T], \mathbb{R}^d)$. In particular, their time $T$-marginals will be equivalent too, which is the claim of the theorem.
\subsection{Proof of the Lemmas}\label{sec:convergence_rates}
We start by proving Lemma \ref{lemma:assumptions_hold}.
\begin{proof}
	The forward drifts are $\beta(x) = 0$, $\beta(x) = -\frac{\alpha}{2}x$ and $\beta(x, v) = (v, -x -2v)$ for the Brownian Motion, the OU-Process and the Critically Damped Langevin Dynamics respectively. In particular, these are all linear maps and therefore fulfil conditions $(i)$ and $(ii)$ of Assumption \ref{assumptions}. 
	
	We show in Appendix \ref{sec:gaussian_transition_kernels} that $\log p_t$ is $C^\infty$ in $t$ and $x$ for $t > 0$.  Therefore we can integrate $p_t$ and its derivative over compact sets, implying that condition $(iii)$ holds. Furthermore, the Hessian w.r.t. $(x,t)$ is continuous and obtains its maximum and minimum on the compact set $[S,T]\times B_N$, where $B_N$ is the ball of diameter $N$ around the origin. Therefore the gradient $\nabla \log p_t$ is Lipschitz on $[S,T]\times B_N$, which proves $(iv)$.
\end{proof}

We now prove Lemma \ref{lemma:optimal_mean_cov_bm_model}.
\begin{proof}
	We have that $X_0 \sim \mu_\text{data}$. Denote the mean and covariance of $\mu_\text{data}$ by $a$ and $C$ respectively. We define
	\[
	n_t = \mathcal{N}(m_t, V_t)
	\]
	for some functions $m_t$ and $V_t$. If $V_t$ would not have full rank, $n_t$ would be a degenerate distribution. Since $p_t > 0$ almost everywhere for $t > 0$, the KL divergence from $p_t$ to $n_t$ would be infinite. We can therefore restrict $V_t$ to be an invertible matrix. We denote the entropy of $p$ by $H$, $H(p) = -\int \log(p) p \mathd x$.
	
	We can write $X_t$ as $X_t = X_0 + \sqrt{t} Z$ where $Z \sim \mathcal{N}(0,I_d)$. Now for the KL-divergence at time $t$ it holds that
	\begin{align*}
		KL(p_t | n_t) &= -H(p_t) - \int \log(n_t(x)) p_t(x) \mathd x \\
		&= -H(p_t) + \frac{d}{2}\log(2\pi) + \frac{1}{2} \log\left(\det(V_t)\right) + \frac{1}{2} \int (x-m_t)^T V_t^{-1} (x-m_t) ~ p_t(x) \mathd x.
	\end{align*}
	We compute
	\begin{align*}
		&~ \int (x-m_t)^T V_t^{-1} (x-m_t) ~ p_t(x) \mathd x\\
		=&~\mathbb{E}[(X_t - m_t)^T V_t^{-1} (X_t - m_t)] 
		= \mathbb{E}[(X_0 + \sqrt{t}Z - m_t)^T V_t^{-1} (X_0 + \sqrt{t}Z - m_t)] \\
		=&~\mathbb{E}[(X_0 - a)^T V_t^{-1} (X_0 - a)] +(a - m_t)^T V_t^{-1} (a - m_t)+ t \mathbb{E}[Z^T V_t^{-1} Z] \\
		=&~\text{tr}(C V_t^{-1}) + t~\text{tr}(I_d V_t^{-1}) + (a - m_t)^T V_t^{-1} (a - m_t).
	\end{align*}
	We added and subtracted $a$ in the second equality. In the last equality we used that $\mathbb{E}[X^T A X] = \text{tr}(\text{Cov}(X)A)$ if $X$ is a centred random variable.
	Substituting  this into our prior calculation, we arrive at
	\begin{align*}
		&KL(p_t | n_t) \\
		=~&-H(p_t) + \frac{1}{2} \log\left(\det(2\pi V_t)\right) + \frac{1}{2}\left(\text{tr}((C + t I_d) V_t^{-1}) + (a - m_t)^T V_t^{-1} (a - m_t)\right).
	\end{align*}
	The mean function $m_t$ only appears in the last term. Therefore, we decrease the KL-divergence by setting $m_t = a$. We now optimize for the covariance matrix $V_t$. The $H(p_t)$ term is independent of $V_t$. Consequently, to optimize the $KL(p_t | n_t)$ with respect to $V_t$, it is enough to optimize
	\begin{align*}
		L(V_t^{-1}) &= \log\left(\det(V_t)\right) + \text{tr}((C + t I_d) V_t^{-1})\\
		&= -\log\left(\det(V_t^{-1})\right) + \text{tr}((C + t I_d) V_t^{-1}), 
	\end{align*}
	where we used that $\text{det}(A^{-1}) = \text{det}(A)^{-1}$ to rewrite the loss in terms of $V_t^{-1}$. Since $-\log(\det(\cdot))$ is convex on the positive semidefinite matrices and the other summands are linear, $L$ is convex in $V_t^{-1}$. We take the gradient of $L$ with respect to the Frobenius inner product and obtain 
	\begin{align*}
		\nabla_{V_t^{-1}} L(V_t^{-1}) = -V_t + C + t I_d.
	\end{align*}
	Setting the gradient to $0$ has the unique solution
	\[
	V_t = C + t I_d,
	\]
	which is a local and global minimum due to the convexity of $L$.
	
	If we restrict $V_t$ to be of the form $V_t = v_t I_d$, the loss  becomes
	\[
	L(v_t^{-1}) = -d\log(v_t^{-1}) + v_t^{-1} \text{tr}(C) + v_t^{-1}td.
	\]
	By setting the derivative of this to zero, we arrive at
	\[
	v_t = \text{tr}(C)/d + t.
	\]
\end{proof}

We now proceed to prove Lemma \ref{lemma:kl_bm_bound}. For that we first need the following proposition.

\begin{proposition}\label{prop:kl_goes_to_0}
	Let $p_t$ be the time $t$-marginal of a Brownian motion started in $\mu_\text{data}$. Assume that the support of $\mu_\text{data}$ is contained in a ball of radius $M$. Let $m_t = \mathbb{E}[\mu_\text{data}]$ and $C_t = \text{Cov}[\mu_\text{data}] + t I_d$ be the optimal mean and covariance operator from Lemma \ref{lemma:optimal_mean_cov_bm_model}. Denote $n_t = \mathcal{N}(m_t, V_t)$. Then $KL(p_t | n_t) \to 0$ as $t \to \infty$.
\end{proposition}
\begin{proof}
	For $t > 0$, the KL-divergence is given by
	 \begin{align}\label{equ:kl_divergence_written_out}
	 	\begin{split}
	 	KL(p_t | n_t) =&-H(p_t) + \frac{1}{2} \log\left(\det(2\pi V_t)\right) + \frac{1}{2}\text{tr}((C + t I_d) V_t^{-1}) \\
	 	=~& -H(p_t) + \frac{1}{2} \log\left(\det(2\pi V_t)\right) + \frac{d}{2} \\
	 	=~& -H(p_t) + \frac{d}{2}\log\left(2 \pi\right) + \frac{1}{2} \log\left(\prod_{i=1}^d (c_i + d)\right) + \frac{d}{2}.
	 	\end{split}
	 \end{align}
 	We used that
 	\begin{equation}\label{equ:determinant_c_tid}
 		\det(C + t I_d) = \prod_{i=1}^{d} (c_i + t),
 	\end{equation}
 	where $c_i$ are the eigenvalues of $\text{Cov}[\mu_\text{data}]$.

 	Now,
	\begin{align*}
		-H(p_t) = \mathbb{E}_{X_t}[\log p_t(X_t)] 
		= \mathbb{E}_{X_t}\left[\log\left(\mathbb{E}_{X_0}\left[(2\pi t)^{-d/2}\exp\left(-\frac{1}{2t}\lVert X_t - X_0 \rVert^2\right)\right]\right)\right].
	\end{align*}
	We bound
	\begin{align*}
		\exp\left(-\frac{1}{2t}\lVert X_t - X_0 \rVert^2\right) \le \exp\left(-\frac{1}{2t}(\lVert X_t \rVert^2 - 2\lVert X_t \rVert M)\right).
	\end{align*}
	We write $X_t$ as $X_t = X_0 + \sqrt{t} Z$ where $Z \sim \mathcal{N}(0, I_d)$. Then,
	\begin{align}\label{equ:appendix_intermediate_upper_bound}
		\begin{split}
		&\log\left(\mathbb{E}_{X_0}\left[(2\pi t)^{-d/2}\exp\left(-\frac{1}{2t}\lVert X_t - X_0 \rVert^2\right)\right]\right) \\
		\le& \log((2\pi t)^{-d/2}) - \frac{1}{2t}(\lVert X_t \rVert^2 - 2 \lVert X_t \rVert M) \\
		\le& -\frac{d}{2}\log(2\pi) -\frac{1}{2} \log(t^d) - \frac{1}{2t}\lVert X_t \rVert^2 + \frac{1}{t} (\lVert M \rVert + \sqrt{t}\lVert Z \rVert) M.
		\end{split}
	\end{align}
	We use that
	\begin{align*}
		\mathbb{E}[\lVert X_t \rVert^2] = \mathbb{E}[\lVert X_0 \rVert^2] + t \mathbb{E}[\lVert Z \rVert^2] = \mathbb{E}[\lVert X_0 \rVert^2] + t \mathbb{E}[\lVert Z \rVert^2] =  m_0^2 + v_0 + td,
	\end{align*}
	where $v_0 = \mathbb{E}[\lVert X_0 - \mathbb{E}[\mu_\text{data}] \rVert^2] = \text{tr}(C_0)$. Since $Z$ is centred and independent of $X_t$, the cross-term $\mathbb{E}[\langle X_t, Z \rangle] = 0$ vanishes. We now take the expectation over the right hand side of \eqref{equ:appendix_intermediate_upper_bound} to obtain
	\begin{align*}
		\mathbb{E}[-\frac{1}{2t}(\lVert X_t \rVert^2 - \lVert X_t \rVert M)] 
		\le \frac{1}{2t}\left(-td - m_0 - v_0 + 2\sqrt{t}\mathbb{E}[\lVert Z \rVert]M + 2M^2\right).
	\end{align*}
	Putting it all together we have that
	\begin{align*}
		KL(p_t | n_t) \le \frac{1}{2}\log\left(\frac{\prod_{i=1}^{d}(c_i + t)}{t^d}\right) + \frac{1}{2t}\left(-m_0-v_0 + 2\sqrt{t}\mathbb{E}[\lVert Z \rVert]M + 2M^2\right) \to 0,
	\end{align*}
	as $t \to \infty$.
\end{proof}

We are now ready to prove Lemma \ref{lemma:kl_bm_bound}.

\begin{proof}
	We again use \eqref{equ:kl_divergence_written_out}. We now take the derivative of the KL-divergence using De Bruijn's identity for multivariate random variables (see \citep{rioul2010information}, \citep{costa1984similarity}):
	\begin{align}
		\frac{\mathd }{\mathd t}H(p_t) = \frac{1}{2}\mathbb{E}_{p_t}[\lVert \nabla \log p_t(X) \rVert^2] = \frac{1}{2t^2}\mathbb{E}[\lVert X_t - \mathbb{E}[X_0 | X_t]\rVert^2],
	\end{align}
	where we applied Lemma \ref{lemma:score_form}. Since $\mathbb{E}[X_0 | X_t]$ is the $L^2$-orthogonal projection of $X_0$ to the $\sigma(X_t)$-measurable random variables and $X_t$ is $\sigma(X_t)$-measurable, we have that
	\begin{align*}
		\frac{1}{2t^2}\mathbb{E}[\lVert X_t - \mathbb{E}[X_0 | X_t]\rVert^2] = \frac{1}{2t^2}\left(\mathbb{E}[\lVert X_t - X_0\rVert^2] - \mathbb{E}[\lVert X_0 - \mathbb{E}[X_0 | X_t]\rVert ^2]\right) = \frac{d}{2t} - \frac{A_t}{2t^2}.
	\end{align*}
	For $A_t = \mathbb{E}[\lVert X_0 - \mathbb{E}[X_0 | X_t]\rVert ^2]$ it holds that
	\[
	0 \le A_t \le \mathbb{E}[\lVert X_0 \rVert ^2] = \text{tr}(\text{Cov}(\mu_\text{data})),
	\]
	since the conditional expectation is a contraction in $L^2$.  We use \eqref{equ:determinant_c_tid} again to arrive at
	\begin{align}\label{equ:kl_h_inequality}
		&\frac{\mathd }{\mathd t} KL(p_t | n_t) = -\frac{1}{2}\left(\frac{d}{t} - \frac{A_t}{t^2}\right) + \frac{1}{2} \sum_{i=1}^d \frac{1}{t+c_i} \ge - \frac{1}{2}\frac{d}{t} + \frac{1}{2} \sum_{i=1}^d \frac{1}{t+c_i} =: h'(t),
	\end{align}
	where $c_i$ are the eigenvalues of $\text{Cov}[\mu_\text{data}] + t I_d$.
	Integrating the right hand side we obtain
	\[
	h(t) = \frac{1}{2}\log\left(\frac{\prod_{i=1}^d (c_i + t)}{t^d}\right)
	\]
	as a possible antiderivative.
	Define $g(t) = h(t) - KL(p_t | n_t)$.
Let us now assume that the support of $\mu_\text{data}$ is contained in a ball of radius $M$. Then, by Proposition \ref{prop:kl_goes_to_0}, $\lim_{t \to \infty} KL(p_t | n_t)  = \lim_{t \to \infty} h(t) = 0$ and therefore $\lim g(t) \to 0$.
Equation \eqref{equ:kl_h_inequality} implies that $g'(t) \le 0$ for all $t > 0$. 
Assume that there is a $s$ and an $\epsilon > 0$ such that
$g(s) \le -\epsilon < 0$. Since $g'(t) \le 0$ for all $t > 0$ we can then conclude that $g(t) \le -\epsilon < 0$ for all $t \ge s$. In particular, $g(t) \not\to 0$ as $t \to \infty$, which is a contradiction. Therefore the statement of the Lemma holds for $\mu_\text{data}$.

Now let $\mu_\text{data}$ be any measure. Let $X$ be a random variable that is distributed according to $\mu_\text{data}$ and $Z$ be a standard normal random variable, $Z \sim \mathcal{N}(0, I_d)$. Then $X + \sqrt{t} Z$ has distribution $p_t$. 
We define $X^M := X 1_{\{\lVert X \rVert \le M\}}$ and $X^M_t = X^M + \sqrt{t} Z$. Denote the distribution of $X^M_t$ by $p_t^M$. The distribution of $X^M$ is supported on a ball of radius $M$. Furthermore, $X^M + \sqrt{t}Z \to X + \sqrt{t} Z$ almost surely and therefore $p_t^M \to p_t$ weakly for any $t$. By lower semi-continuity of the KL-divergence,
\begin{equation}\label{equ:kl_upper_bound_for_M}
	KL(p_t | n_t) \le \liminf_{M \to \infty} KL(p_t^M | n_t) \le \liminf_{M \to \infty}  \frac{1}{2}\log\left(\frac{\prod_{i=1}^d (c_i^M + t)}{t^d}\right),
\end{equation}
where $c_i^M$ are the eigenvalues of $\text{Cov}[\mu_\text{data}^M] + t I_d$.
Since $|X^M| \le |X|$ and $X$ has first and second moments one can apply dominated convergence to see that the covariance matrices converge, $\text{Cov}(\mu_\text{data}^M) \to \text{Cov}(\mu_\text{data})$. Especially
\[
\lim_{M \to \infty} \prod_{i=1}^d (c_i^M + t) = \lim_{M \to \infty} \det(\text{Cov}(\mu_\text{data}^M)) 
\to \det(\text{Cov}(\mu_\text{data})) = \prod_{i=1}^d (c_i + t).
\]
Combining this with \eqref{equ:kl_upper_bound_for_M} concludes the proof.
\end{proof}

\end{document}